\documentclass[final,hidelinks,onefignum,onetabnum]{siamart251216}
\usepackage{graphicx}
\usepackage{enumitem}
\usepackage{amsfonts,amssymb}
\usepackage{mathtools}
\usepackage{booktabs}
\usepackage{algorithmicx}
\usepackage[noend]{algpseudocode}
\usepackage{placeins}
\usepackage{multirow}
\makeatletter
\newcommand*\bigcdot{\mathpalette\bigcdot@{.5}}
\newcommand*\bigcdot@[2]{\mathbin{\vcenter{\hbox{\scalebox{#2}{$\m@th#1\bullet$}}}}}
\makeatother
\newsiamremark{remark}{Remark}
\newsiamremark{example}{Example}

\newcommand{\E}{\mathbb{E}}

\DeclareMathOperator{\KL}{D_{\mathrm{KL}}}

\newcommand{\Enc}{\operatorname{Enc}}
\graphicspath{{fig/}}

\headers{Amortized Filtering and Smoothing with Conditional Flows}
{T. Cui, X. Feng, C. Pei, X. Wan, and T. Zhou}

\title{Amortized Filtering and Smoothing with Conditional Normalizing Flows}

\author{Tiangang Cui\thanks{School of Mathematics and Statistics, The University of Sydney, Camperdown, NSW 2006, Australia
  (\email{tiangang.cui@sydney.edu.au}).}
\and Xiaodong Feng\thanks{Institute of Advanced Study, Beijing Normal-Hong Kong Baptist University, Zhuhai 519087, China
  (\email{xiaodongfeng@bnbu.edu.cn}).}
\and Chenlong Pei\thanks{SKLMS \& Institute of Computational Mathematics and Scientific/Engineering Computing, Academy of Mathematics and Systems Science, Chinese Academy of Sciences
  (\email{peichenlong@amss.ac.cn}, \email{tzhou@lsec.cc.ac.cn}).}
\and Xiaoliang Wan\thanks{Department of Mathematics and Center for Computation and Technology, Louisiana State University, Baton Rouge 70803, USA
  (\email{xlwan@lsu.edu}).}
\and Tao Zhou\footnotemark[3]}

\ifpdf
\hypersetup{
  pdftitle={Amortized Filtering and Smoothing with Conditional Normalizing Flows},
  pdfauthor={Tiangang Cui, Xiaodong Feng, Chenlong Pei, Xiaoliang Wan, and Tao Zhou}
}
\fi

\begin{document}
	\maketitle

			\begin{abstract}
				Bayesian filtering and smoothing are central to data assimilation in nonlinear dynamical systems. Recent advances in deep generative models provide flexible approximations of the associated non-Gaussian posterior distributions. However, several existing approaches require the score to be re-estimated or a transport map to be constructed at each assimilation step. We propose an amortized framework for filtering and smoothing that reuses trained conditional models across observation sequences and assimilation times. The framework jointly learns a shared recurrent summary network and two conditional normalizing flows from simulated state and observation trajectories. The recurrent network represents each observation history by a fixed-dimensional summary that conditions the filtering approximation and, together with the next state, the backward transition approximation. An information-theoretic analysis shows that, under the Markov assumption, a summary sufficient for filtering is also sufficient for the backward kernel. Combining the terminal filtering approximation with the backward kernels yields an approximate joint smoothing distribution. Numerical experiments on advection-diffusion, Burgers, and Lorenz systems demonstrate the accuracy of the proposed filtering and smoothing approximations and characterize the evolution of their errors beyond the training horizon.
			\end{abstract}


	\section{Introduction}\label{sec:intro}

	Bayesian filtering and smoothing are fundamental to data assimilation in dynamical systems, where hidden states are inferred from indirect and noisy observations \cite{law2015data,asch2016data}. Filtering determines the distribution of the current state conditioned on observations available so far, whereas smoothing uses subsequent observations to infer past states \cite{sarkka2023bayesian}. For linear Gaussian state-space models, the Kalman filter provides exact recursions for the conditional mean and covariance \cite{kalman1960new}. For nonlinear systems, local linearization and ensemble Kalman updates provide tractable approximations, but can introduce bias in non-Gaussian settings \cite{evensen2003ensemble,spantini2022coupling}. Particle methods represent more general posterior distributions using weighted samples \cite{gordon1993novel,doucet2001sequential}. However, weight degeneracy can leave only a few particles contributing appreciably to the posterior approximation \cite{bengtsson2008curse,snyder2008obstacles}. Filtering and smoothing remain challenging for high-dimensional nonlinear systems, particularly when posterior approximations must be updated sequentially as new observations arrive.

	Recent work has developed machine learning methods for filtering in high-dimensional nonlinear systems \cite{cheng2023machine,bach2024machine,revach2022kalmannet,tong2026latent}. When the dynamical model and observation operator are available, several methods generate forecast ensembles with the dynamical model and learn the transformation from forecast to analysis. The enhanced ensemble filters of Bach et al.\ learn ensemble-dependent corrections to Kalman updates \cite{bach2025learning}, while the ensemble neural filter developed within BI-EqNO learns a permutation-equivariant transformation from forecast to analysis ensembles \cite{zhou2024bi}. Generative models and measure transport offer another means of approximating non-Gaussian filtering distributions \cite{huang2024diffda,huynh2025joint,spantini2022coupling,zhao2024tensor,wang2025flow}. Score-based filters approximate the score of the forecast distribution through neural-network training at successive assimilation times \cite{bao2024score} or estimate it directly from forecast ensembles \cite{bao2024ensemble}. The resulting score approximation is combined with the observation model to generate posterior samples through reverse diffusion. Transport-based filters construct maps from a reference or forecast distribution to an approximation of the filtering distribution \cite{taghvaei2020optimal,spantini2022coupling,zhao2024tensor}. Normalizing flows parameterize such maps as invertible transformations, supporting direct sampling and tractable density evaluation \cite{kobyzev2020normalizing,papamakarios2021normalizing}.

	In many applications, the dynamical model and observation operator are not explicitly available, and inference must rely on simulated state and observation data \cite{cranmer2020frontier,deistler2025simulation,zammit2025neural}. One approach is to learn a latent state-space model whose filtering distributions admit analytical Gaussian updates, as in deep Bayesian filtering and the flow-based Bayesian filter \cite{tarumi2025deep,wang2025flow}. Simulation-based inference also allows conditional state distributions to be learned directly from paired trajectories, without evaluating transition or observation densities. Delecki et al.\ use normalizing flows conditioned on encoded observation histories \cite{delecki2023deep}, while Al-Jarrah et al.\ learn filtering transport maps conditioned on a fixed observation window \cite{aljarrah2024data}. Once estimated, the same conditional approximations can be applied across observation sequences and assimilation times.

	Joint smoothing targets the posterior distribution of an entire state trajectory given all available observations. A collection of filtering distributions does not by itself determine the temporal dependence among the inferred states, which is essential for uncertainty quantification of path-dependent quantities. Score-based data assimilation and Appa learn trajectory priors offline \cite{rozet2023score,andry2025appa}. During inference, the observation model supplies guidance terms in the reverse diffusion process to generate trajectories conditioned on the observed sequence. Chagneux et al.\ exploit the Markov structure to learn joint smoothing distributions through amortized variational inference, using an objective that requires evaluation of the state-space model densities \cite{chagneux2024additive}.

	In this work, we consider Bayesian filtering and joint smoothing based on paired simulated state and observation trajectories, without assuming access to transition and observation densities. Existing learning-based approaches incorporate data-driven approximations into filtering and smoothing either by learning components of a recursive model-based update or by representing state and trajectory distributions. Our method directly approximates the filtering distributions together with the backward kernels appearing in the Markov factorization of the joint smoothing posterior. These conditional distributions provide a common probabilistic representation of filtering and joint smoothing and are learned jointly from simulated trajectories.

	Our contributions are summarized below.
	\begin{enumerate}[topsep=6pt, partopsep=0pt, leftmargin=\parindent]
		\item We propose an amortized method for Bayesian filtering and joint smoothing that learns filtering distributions and backward kernels from simulated state and observation trajectories. These conditional models are represented by two normalizing flows trained jointly with a shared recurrent summary network. The learned models can be used for inference across observation sequences and assimilation times.
		\item Under the Markov assumption, we show that the information lost by compressing the observation history is no greater for the backward kernel than for filtering. Consequently, any summary sufficient for filtering is also sufficient for the backward kernel. This result provides a theoretical basis for using a shared observation summary in the two conditional models.
		\item Numerical experiments on advection-diffusion, Burgers, and Lorenz systems assess the accuracy of the filtering and smoothing approximations across state dimensions, spatial resolutions, and assimilation times.
	\end{enumerate}

	The remainder of this paper is organized as follows. Section~\ref{sec:problem} introduces the Bayesian filtering and smoothing setting and reviews conditional normalizing flows and recurrent neural networks. Section~\ref{sec:method} presents the jointly amortized method based on shared summary statistics, a forward flow, and a backward flow, together with the theoretical results and the training and inference procedures. Section~\ref{sec:experiments} presents the numerical experiments. Section~\ref{sec:conclusion} concludes the paper. Proofs and additional numerical details are collected in the appendices.

\section{Problem setup and preliminaries}\label{sec:problem}
	\subsection{Problem setup}
	We consider a discrete-time stochastic dynamical system described by the state-space model
	\begin{equation}\label{SSM}
u_t=f(u_{t-1},\epsilon_{u,t}),\qquad y_t=h(u_t,\epsilon_{y,t}),
\end{equation}
	where $t \in \mathbb{N}^+$ denotes discrete time, $u_t \in \mathbb{R}^{n_u}$ is the
	hidden state vector, and $y_t \in \mathbb{R}^{n_y}$ is the corresponding observation.
	The state $u_t$ evolves according to the possibly nonlinear transition function $f$,
	while $y_t$ is generated by the measurement function $h$. The initial state $u_0$ has a prescribed distribution.
	Here, $\epsilon_{u,t}$ and $\epsilon_{y,t}$ denote the process and observation noises, respectively; the two noise sequences are independent of each other and of $u_0$, and each sequence is independent and identically distributed over time.
	For notational simplicity, we write $u_{1:t} \coloneqq  (u_1,\dots,u_t)$ for the state trajectory and $y_{1:t} \coloneqq  (y_1,\dots,y_t)$ for the observation sequence.
	The state-space model \eqref{SSM} induces the Markov factorization
\begin{equation*}
p(u_{1:t},y_{1:t})=
p(u_1)\,p(y_1\mid u_1)\prod_{k=2}^t p(u_k\mid u_{k-1})\,p(y_k\mid u_k),
\end{equation*}
where $p(u_k\mid u_{k-1})$ and $p(y_k\mid u_k)$ denote the state transition and observation distributions, respectively. At each time $t$, we consider the following inference problems.
	\begin{itemize}[topsep=6pt, partopsep=0pt, leftmargin=\parindent]
		\item \textbf{Filtering.}
		Estimate the conditional distribution of the current state $u_t$ given all observations up to time $t$. The target density is $p(u_t \mid y_{1:t})$.
		\item \textbf{Joint smoothing.}
		Infer the entire state trajectory conditioned on all available observations. The target density is $p(u_{1:t} \mid y_{1:t})$.
		Both the dimension of the state trajectory and the length of the observation history increase with $t$, making this problem progressively more challenging.
		
		\item \textbf{Smoothing marginals.}
		Estimate the conditional distribution of a past state $u_k$, with $k<t$, given all observations up to time $t$. The target density is $p(u_k\mid y_{1:t})$.
	\end{itemize}	
We approximate the filtering and joint smoothing distributions defined above. Samples from the smoothing marginals are obtained by taking the time-$k$ components of samples from the joint approximation.

The available data consist of $N$ paired state and observation trajectories $\{u_{1:T}^i,y_{1:T}^i\}_{i=1}^N$, generated by a simulator for the state-space model \eqref{SSM}. We learn the posterior approximations from these trajectories and apply them at assimilation times within and beyond the training horizon. We next introduce conditional normalizing flows and recurrent neural networks, which define the conditional density approximations and the observation summaries, respectively.

	\subsection{Conditional normalizing flows}\label{CNF}
Let $U\in\mathbb R^{d_u}$ be a random vector, and let $s\in\mathbb R^{d_s}$ denote a conditioning variable. For each fixed $s$, let $p(\cdot\mid s)$ denote the unknown conditional density of $U$ given $s$. Our goal is to learn a conditional flow model with parameters $\theta$ such that $p_\theta(\cdot\mid s)$ approximates $p(\cdot\mid s)$ over the range of conditioning variables of interest.

Let $Z\in\mathbb R^{d_u}$ be an analytically tractable reference random variable with known density $p_Z$. For each $s$, a conditional normalizing flow defines a differentiable invertible mapping
\begin{equation*}
f_\theta(\cdot;s):\mathbb R^{d_u}\to\mathbb R^{d_u},
\qquad
z=f_\theta(u;s).
\end{equation*}
The transformed reference variable $U_\theta=f_\theta^{-1}(Z;s)$ has the conditional density
\begin{equation*}
p_\theta(u\mid s)
=p_Z\bigl(f_\theta(u;s)\bigr)
\left|\det\nabla_u f_\theta(u;s)\right|
\end{equation*}
by the change-of-variables formula. To represent complicated conditional distributions, $f_\theta$ is constructed by composing structured bijections with tractable Jacobian determinants. In this work, we use the conditional KRnet construction~\cite{feng2022solving,he2025adaptive}.

Given paired samples $\{(u^{(n)},s^{(n)})\}_{n=1}^N$, the parameter $\theta$ is obtained by maximizing the conditional log-likelihood
\begin{equation*}
\mathcal L_N(\theta)
=\frac1N\sum_{n=1}^N
\log p_\theta\bigl(u^{(n)}\mid s^{(n)}\bigr).
\end{equation*}
In this formulation, $s$ is a conditioning variable of fixed dimension. In sequential inference, the available information at time $t$ is the full observation history $y_{1:t}$, whose length grows with $t$. We therefore introduce recurrent neural networks to convert this variable-length sequence into fixed-dimensional summary statistics on which the conditional flows depend.

\subsection{Recurrent neural networks}\label{RNN}
We use a multilayer long short-term memory (LSTM) network~\cite{yu2019review} to map the observation history $y_{1:t}$ to fixed-dimensional summary statistics for the conditional normalizing flows. At each time $t$, the hidden and cell states are updated from the current observation $y_t$ and their values at time $t-1$.

Let $L$ denote the number of LSTM layers, and let $h_t^{(\ell)},c_t^{(\ell)}\in\mathbb R^{d_h}$ denote the hidden and cell states of layer $\ell$. With $\boldsymbol h_t=(h_t^{(1)},\ldots,h_t^{(L)})$ and $\boldsymbol c_t=(c_t^{(1)},\ldots,c_t^{(L)})$, the multilayer recurrence is written as
\begin{equation*}
\bigl(\boldsymbol h_t,\boldsymbol c_t\bigr)
=\mathcal G
\bigl(y_t,\boldsymbol h_{t-1},\boldsymbol c_{t-1}\bigr),
\qquad t\geq1,
\end{equation*}
where $\mathcal G$ denotes the multilayer LSTM update and the initial states are set to $\boldsymbol h_0=\boldsymbol c_0=\mathbf 0$. Within each update, the first layer takes $y_t$ as input, while each higher layer takes the hidden state produced by the layer below it. Repeated application of $\mathcal G$ determines the hidden state $h_t^{(L)}$ of the final LSTM layer from the observation history $y_{1:t}$.
The summary statistic is then defined by the affine transformation
\begin{equation*}
s_t=W_s h_t^{(L)}+b_s,
\end{equation*}
where $W_s\in\mathbb R^{d_s\times d_h}$ and $b_s\in\mathbb R^{d_s}$ are trainable parameters. Although the dimension of $y_{1:t}$ increases with $t$, the resulting summary statistic $s_t$ remains in the fixed-dimensional space $\mathbb R^{d_s}$.

\section{Methodology}\label{sec:method}

We construct an amortized approximation of the filtering and joint smoothing distributions using a recurrent observation summary and two conditional normalizing flows. The forward flow approximates the filtering distribution from the summary, while the backward flow combines the same summary with the next state to approximate a backward kernel. Their composition defines a joint smoothing distribution from which state trajectories and smoothing marginals can be sampled. We first introduce the two conditional models and examine the information required by their shared summary. We next present joint training and posterior sampling, and then analyze the resulting posterior approximation errors.

\subsection{Forward conditional flow for amortized filtering}\label{sec:amortized-filtering}

The filtering distribution $p(u_t\mid y_{1:t})$ depends on an observation history $y_{1:t}$ whose dimension increases with $t$. We represent this history by the fixed-dimensional summary
\begin{equation*}
    s_t=\Enc(y_{1:t};\psi)\in\mathbb R^{d_s},
\end{equation*}
where $\psi$ includes the parameters of the recurrent network and the affine transformation introduced in Section~\ref{RNN}. A conditional normalizing flow, referred to as the forward flow, defines the approximation
\begin{equation}\label{eq:filtering_model}
    p(u_t\mid y_{1:t})
    \approx p_{\theta_1,\psi}(u_t\mid s_t),
    \qquad t\geq1.
\end{equation}
The same parameters $(\theta_1,\psi)$ are used across observation sequences and assimilation times.
At a fixed assimilation time $t$, consider the population conditional cross-entropy
\begin{equation*}
    \mathcal L_F^{(t)}(\theta_1,\psi)
    =\E_{u_t,y_{1:t}}
    \left[-\log p_{\theta_1,\psi}(u_t\mid s_t)\right],
\end{equation*}
where the expectation is taken under the joint law of the state-space model.

Assume that the conditional densities below exist and that the stated entropies and divergences are finite. Let $I$ denote conditional mutual information.

\begin{proposition}[Equivalent filtering objective]\label{prop:filtering-decomposition}
Minimizing the population loss $\mathcal L_F^{(t)}(\theta_1,\psi)$ over $(\theta_1,\psi)$ is equivalent to minimizing
\begin{equation}\label{eq:filtering-loss-decomposition}
    I(u_t;y_{1:t}\mid s_t)
    +\E_{s_t}\!\left[
    \KL\!\left(
    p(\,\cdot\mid s_t)
    \,\middle\|\,
    p_{\theta_1,\psi}(\,\cdot\mid s_t)
    \right)\right].
\end{equation}

\end{proposition}

Proposition~\ref{prop:filtering-decomposition} distinguishes the information lost through observation compression from the conditional density approximation error of the forward flow. The first term penalizes the loss of state information when compressing $y_{1:t}$ into $s_t$. The second term measures the error in approximating $p(u_t\mid s_t)$, which describes the state uncertainty after compression. The proof is given in Appendix~\ref{appendix:loss-decompositions}.

\subsection{Backward factorization and shared summary}\label{sec:amortized-smoothing}

Under the Markov state-space model, the joint smoothing distribution admits the backward factorization
\begin{equation}\label{eq:path_factorization}
    p(u_{1:t}\mid y_{1:t})
    =p(u_t\mid y_{1:t})
    \prod_{k=1}^{t-1}p(u_k\mid u_{k+1},y_{1:k}),
    \qquad t\geq2.
\end{equation}
This factorization reduces joint smoothing to a terminal filtering distribution and a sequence of backward kernels. In this representation, observations after time $k$ influence $u_k$ through $u_{k+1}$, whose smoothing distribution is obtained by continuing backward from the terminal filtering distribution.

We introduce a second conditional normalizing flow, referred to as the backward flow, to approximate
\begin{equation}\label{eq:backward_model}
    p(u_k\mid u_{k+1},y_{1:k})
    \approx p_{\theta_2,\psi}(u_k\mid u_{k+1},s_k).
\end{equation}
The backward flow uses the same observation summary $s_k$ as the forward flow, together with the next state $u_{k+1}$. The two flows therefore define the approximate joint smoothing density 
\begin{equation}\label{eq:learned-smoothing-path}
    p_{\theta_1,\theta_2,\psi}^{\mathrm{smoothing}}(u_{1:t}\mid s_{1:t})
    \coloneqq
    p_{\theta_1,\psi}(u_t\mid s_t)
    \prod_{k=1}^{t-1}p_{\theta_2,\psi}(u_k\mid u_{k+1},s_k).
\end{equation}
The forward and backward flows each act on $\mathbb R^{n_u}$, independently of the trajectory length. Figure~\ref{fig:method} summarizes the shared summary network and the two conditional flows.

\begin{figure}[t]
 \centering
 \includegraphics[width=0.9\linewidth]{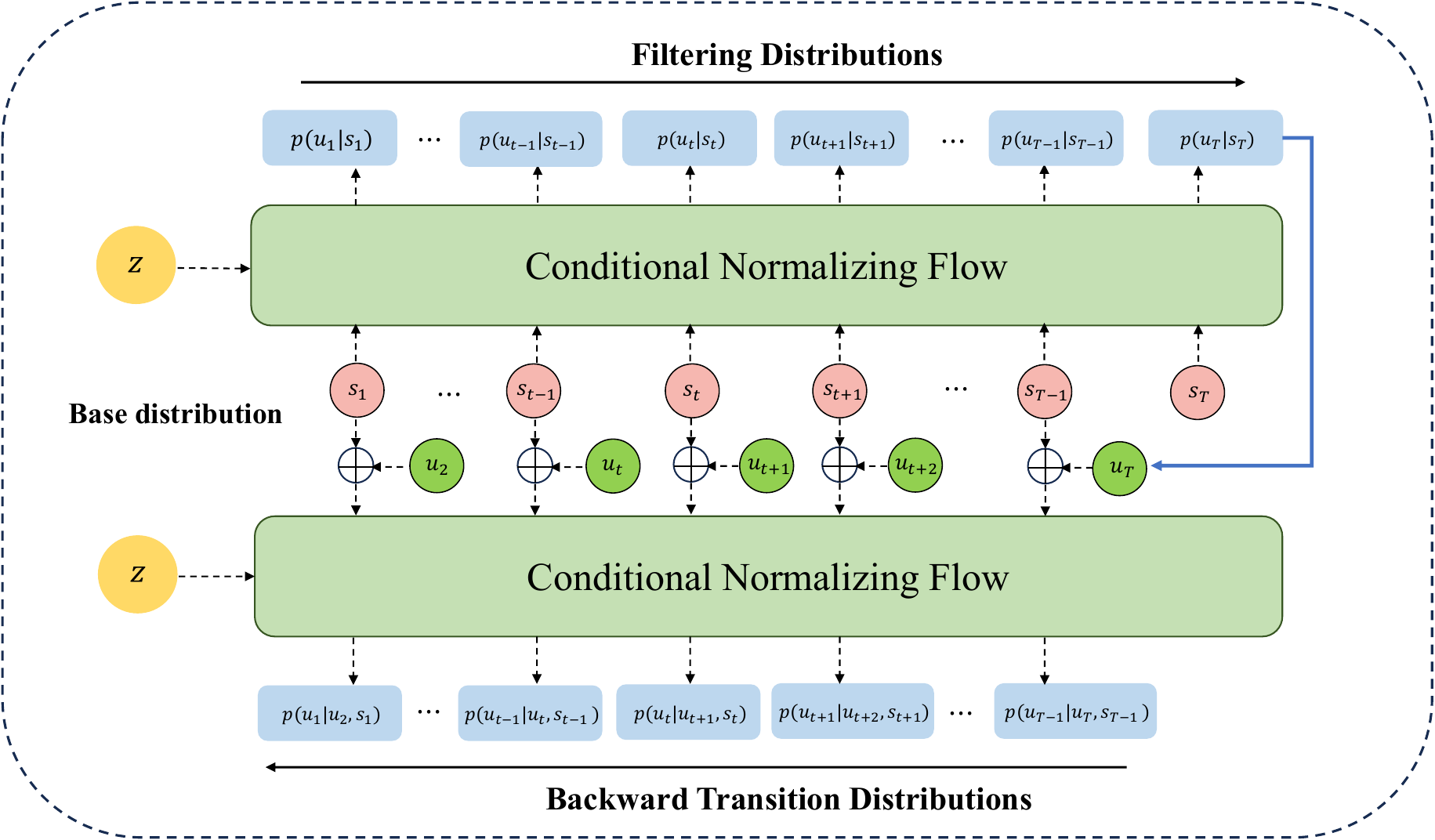}
 \caption{Shared recurrent summary network and conditional flows. The summary network maps the observation history $y_{1:t}$ to a fixed-dimensional summary statistic $s_t$. Conditioned on this summary, the forward flow approximates $p(u_t\mid y_{1:t})$ and the backward flow approximates $p(u_k\mid u_{k+1},y_{1:k})$. The terminal forward density and the backward kernels define an approximate joint smoothing distribution.}
 \label{fig:method}
\end{figure}

We next characterize the relation between summaries sufficient for filtering and those sufficient for the backward kernel.

\begin{definition}[Filtering sufficiency]\label{def:filtering-sufficiency}
A measurable summary $s_t$ of $y_{1:t}$ is filtering sufficient at time $t$ if $p(u_t\mid y_{1:t})=p(u_t\mid s_t)$ almost surely. 
Equivalently, $I(u_t;y_{1:t}\mid s_t)=0$. A sequence $\{s_t\}_{t\geq1}$ is filtering sufficient if this condition holds at every time.
\end{definition}

\begin{proposition}[Information loss for the backward kernel]\label{prop:shared-summary}
Under the Markov state-space model, let $s_k$ be a measurable function of $y_{1:k}$. Then
\begin{equation}\label{eq:shared-information-relation}
I(u_k;y_{1:k}\mid s_k)
    ={}I(u_{k+1};y_{1:k}\mid s_k) +I(u_k;y_{1:k}\mid u_{k+1},s_k).
\end{equation}
Consequently,
\begin{equation*}
    I(u_k;y_{1:k}\mid u_{k+1},s_k)
    \leq I(u_k;y_{1:k}\mid s_k).
\end{equation*}
If $s_k$ is filtering sufficient at time $k$, then
\begin{equation}\label{eq:backward-sufficiency}
    p(u_k\mid u_{k+1},y_{1:k})
    =p(u_k\mid u_{k+1},s_k)
\end{equation}
almost surely. 
\end{proposition}

Proposition~\ref{prop:shared-summary} shows that the information loss for the backward kernel is bounded by that for filtering. In particular, if $s_k$ is filtering sufficient, then the observation history enters the backward kernel only through $s_k$, with $u_{k+1}$ retained as a conditioning variable. The same observation summary can therefore be used to represent both conditional distributions. The identity follows from the chain rule for conditional mutual information and the Markov property; the proof is given in Appendix~\ref{appendix:shared-summary-information}.

The linear-Gaussian model provides an exact finite-dimensional example of such a shared recursive summary.

\begin{example}\label{ex:linear-gaussian-statistic}
Consider the time-homogeneous linear-Gaussian model
\begin{equation*}
    u_t=Mu_{t-1}+\eta_t,
    \qquad
    y_t=Hu_t+\varepsilon_t,
\end{equation*}
where $u_0\sim\mathcal N(m_0,P_0)$, $\eta_t\sim\mathcal N(0,Q)$, and $\varepsilon_t\sim\mathcal N(0,R)$. The initial state and all noise variables are mutually independent, and $P_0,Q,R\succ0$.
The Kalman filtering distribution is $p(u_t\mid y_{1:t})=\mathcal N(m_t,P_t)$. Its mean and covariance therefore form a filtering-sufficient summary $s_t^\star=(m_t,P_t)$, whose dimension is independent of $t$. The Kalman update gives $s_t^\star=\Phi(s_{t-1}^\star,y_t)$, with the same map $\Phi$ at every assimilation time. By Proposition~\ref{prop:shared-summary}, this summary is also sufficient for the backward kernel. The filtering distribution and the backward kernel are Gaussian and admit affine transports conditioned on $s_t^\star$ and $(u_{t+1},s_t^\star)$, respectively. Thus, a single recursive summary and two conditional affine transports suffice to represent the exact filtering and joint smoothing distributions at every assimilation time. The detailed construction is given in Appendix~\ref{appendix:linear-gaussian-construction}.
\end{example}

Having specified the two conditional models and their shared summary, we now estimate all parameters jointly.

\subsection{Joint training and inference}\label{sec:training-inference}

For simulated trajectories $\{(u_{1:T}^i,y_{1:T}^i)\}_{i=1}^N$ of length $T\geq2$, we jointly train the summary network and the two conditional flows by minimizing
\begin{equation}\label{eq:loss_shared}
\resizebox{0.91\linewidth}{!}{$\displaystyle
\widehat{\mathcal L}(\theta_1,\theta_2,\psi)
    ={}-\frac1{NT}\sum_{i=1}^N\left[
    \sum_{t=1}^T\log p_{\theta_1,\psi}(u_t^i\mid s_t^i)
    +\sum_{t=1}^{T-1}\log p_{\theta_2,\psi}(u_t^i\mid u_{t+1}^i,s_t^i)
    \right].
$}
\end{equation}
This representation displays the forward and backward conditional likelihoods used in training. By the backward factorization \eqref{eq:learned-smoothing-path}, the same loss can be written as
\begin{equation}\label{eq:empirical-path-objective}
\widehat{\mathcal L}
    =-\frac1{NT}\sum_{i=1}^N\Bigg[
    \sum_{t=1}^{T-1}\log p_{\theta_1,\psi}(u_t^i\mid s_t^i)
    +\log p_{\theta_1,\theta_2,\psi}^{\mathrm{smoothing}}(u_{1:T}^i\mid s_{1:T}^i)
    \Bigg].
\end{equation}
The second representation combines the negative log-likelihoods of the intermediate filtering densities and the joint smoothing density.

Given simulated trajectories $\{(u_{1:T}^i,y_{1:T}^i)\}_{i=1}^N$, the recurrent network maps each observation history $y_{1:t}^i$ to a summary $s_t^i$ shared by the two conditional flows. The forward density is conditioned on $s_t^i$, while the backward density is additionally conditioned on the simulated next state $u_{t+1}^i$. We estimate $(\theta_1,\theta_2,\psi)$ jointly by minimizing $\widehat{\mathcal L}$, so that both conditional likelihoods enter the estimation of the shared summary parameters. Algorithm~\ref{alg:train} summarizes the training procedure.

\begin{algorithm}[ht]
 \caption{Joint training for filtering and smoothing}
 \label{alg:train}
 \begin{algorithmic}[1]
  \State \textbf{Input:} simulated trajectories $\{(u_{1:T}^i,y_{1:T}^i)\}_{i=1}^N$, number of epochs $N_e$, minibatch size $N_b$, and initial parameters $(\theta_1,\theta_2,\psi)$.
  \For{$e=1,\ldots,N_e$}
   \For{each minibatch $\mathcal B\subset\{1,\ldots,N\}$ with $|\mathcal B|\leq N_b$}
    \For{$i\in\mathcal B$}
     \State Compute $s_t^i=\Enc(y_{1:t}^i;\psi)$ for $t=1,\ldots,T$.
    \EndFor
    \State Evaluate the minibatch version of $\widehat{\mathcal L}$ in \eqref{eq:loss_shared} and take one optimization step for $(\theta_1,\theta_2,\psi)$.
   \EndFor
  \EndFor
  \State \textbf{Output:} trained parameters $(\theta_1,\theta_2,\psi)$.
 \end{algorithmic}
\end{algorithm}

At inference, we apply the trained model to new observations $y_{1:t}$, with $t$ possibly exceeding the training horizon $T$. The summary network computes $s_1,\ldots,s_t$, and the forward flow generates filtering samples at time $t$. For smoothing, the backward flow extends each terminal sample into a trajectory by drawing $u_k$ conditional on the sampled $u_{k+1}$ and the summary $s_k$. All parameters remain fixed during inference. Algorithm~\ref{alg:predict-smooth} summarizes the procedure.

\begin{algorithm}[ht]
 \caption{Filtering and smoothing inference}
 \label{alg:predict-smooth}
 \begin{algorithmic}[1]
  \State \textbf{Input:} observations $y_{1:t}$, trained parameters $(\theta_1,\theta_2,\psi)$, and number of samples $N_{\mathrm{sample}}$.
  \State Evaluate the recurrent summary network to obtain $s_1,\ldots,s_t$.
  \State Draw $u_t^{(j)}\sim p_{\theta_1,\psi}(\,\cdot\mid s_t)$ for $j=1,\ldots,N_{\mathrm{sample}}$.
  \For{$k=t-1,t-2,\ldots,1$}
   \State Draw $u_k^{(j)}\sim p_{\theta_2,\psi}(\,\cdot\mid u_{k+1}^{(j)},s_k)$ for $j=1,\ldots,N_{\mathrm{sample}}$.
  \EndFor
  \State \textbf{Output:} filtering samples $\{u_t^{(j)}\}_{j=1}^{N_{\mathrm{sample}}}$ and smoothing trajectories $\{u_{1:t}^{(j)}\}_{j=1}^{N_{\mathrm{sample}}}$.
 \end{algorithmic}
\end{algorithm}

The backward recursion in Algorithm~\ref{alg:predict-smooth} generates samples from the joint smoothing density defined in \eqref{eq:learned-smoothing-path}. For each $k<t$, the components $\{u_k^{(j)}\}_{j=1}^{N_{\mathrm{sample}}}$ are therefore samples from its $k$-th marginal, which provides the approximation to $p(u_k\mid y_{1:t})$. These samples can be used to estimate smoothing means, variances, and other posterior quantities by Monte Carlo averaging.

\subsection{Error analysis for filtering and smoothing}\label{sec:posterior-errors}

The population loss associated with the joint objective \eqref{eq:empirical-path-objective} is
\begin{equation*}
    \mathcal L(\theta_1,\theta_2,\psi)
    =\E\!\left[\widehat{\mathcal L}(\theta_1,\theta_2,\psi)\right],
\end{equation*}
where the expectation is taken with respect to the joint law of the simulated training trajectories. The following proposition characterizes the population loss in terms of filtering and smoothing KL errors.

\begin{proposition}[Equivalent joint objective]\label{prop:population-objective}
Minimizing $\mathcal L$ over $(\theta_1,\theta_2,\psi)$ is equivalent to minimizing
\begin{equation*}
\resizebox{0.97\linewidth}{!}{$\displaystyle
    \sum_{\ell=1}^{T-1}\E_{y_{1:\ell}}\!\left[
    \KL\!\left(
    p(u_\ell\mid y_{1:\ell})
    \,\middle\|\,
    p_{\theta_1,\psi}(u_\ell\mid s_\ell)
    \right)\right] +\E_{y_{1:T}}\!\left[
    \KL\!\left(
    p(u_{1:T}\mid y_{1:T})
    \,\middle\|\,
    p_{\theta_1,\theta_2,\psi}^{\mathrm{smoothing}}(u_{1:T}\mid s_{1:T})
    \right)\right].
$}
\end{equation*}
\end{proposition}

The equivalent objective combines the expected filtering errors at times $1,\ldots,T-1$ with the expected joint smoothing error at time $T$. The filtering error at time $T$ is already included in the joint smoothing term through the backward factorization. The proof is given in Appendix~\ref{appendix:joint-objective}.

Fix the model parameters and let $t\geq2$ denote the inference horizon. Conditional on $y_{1:t}$, the following theorem gives an exact decomposition of the joint smoothing KL error and an upper bound for each marginal KL error.
\begin{theorem}[Joint and marginal smoothing errors]\label{thm:path-kl}
Under the Markov state-space model, for any $t\geq2$ and almost every observation sequence $y_{1:t}$,
\begin{gather*}
\scalebox{0.88}{$\displaystyle
\begin{aligned}
\KL\!{}&
\left(
    p(u_{1:t}\mid y_{1:t})
    \,\middle\|\,
    p_{\theta_1,\theta_2,\psi}^{\mathrm{smoothing}}
    (u_{1:t}\mid s_{1:t})
\right)
\\[0.3em]
&=
\KL\!\left(
    p(u_t\mid y_{1:t})
    \,\middle\|\,
    p_{\theta_1,\psi}(u_t\mid s_t)
\right)
+\sum_{\ell=1}^{t-1}
\E_{u_{\ell+1}\mid y_{1:t}}\!\left[
    \KL\!\left(
        p(\,\cdot\mid u_{\ell+1},y_{1:\ell})
        \,\middle\|\,
        p_{\theta_2,\psi}(\,\cdot\mid u_{\ell+1},s_\ell)
    \right)
\right].
\end{aligned}
$}
\end{gather*}
Let $p_{\theta_1,\theta_2,\psi}^{\mathrm{smoothing}}(u_k\mid s_{1:t})$ denote the $k$th marginal of the approximate joint smoothing density. Then, for every $1\leq k\leq t$,
\begin{gather*}
\scalebox{0.88}{$\displaystyle
\begin{aligned}
\KL\!{}&
\left(
    p(u_k\mid y_{1:t})
    \,\middle\|\,
    p_{\theta_1,\theta_2,\psi}^{\mathrm{smoothing}}
    (u_k\mid s_{1:t})
\right)
\\[0.3em]
&\leq
\KL\!\left(
    p(u_t\mid y_{1:t})
    \,\middle\|\,
    p_{\theta_1,\psi}(u_t\mid s_t)
\right)
+\sum_{\ell=k}^{t-1}
\E_{u_{\ell+1}\mid y_{1:t}}\!\left[
    \KL\!\left(
        p(\,\cdot\mid u_{\ell+1},y_{1:\ell})
        \,\middle\|\,
        p_{\theta_2,\psi}(\,\cdot\mid u_{\ell+1},s_\ell)
    \right)
\right].
\end{aligned}
$}
\end{gather*}
\end{theorem}

The theorem describes how approximation errors accumulate along the backward factorization. The joint smoothing KL error is the sum of the terminal filtering error and the backward-kernel errors, where the error at step $\ell$ is averaged with respect to the true smoothing distribution $p(u_{\ell+1}\mid y_{1:t})$. For the marginal at time $k$, the bound involves the terminal filtering error and the $t-k$ backward steps from time $t$ to time $k$; its dependence on $k$ therefore reflects the length of the backward recursion. The proof is given in Appendix~\ref{appendix:joint-marginal-errors}.

At the training horizon $t=T$, the joint smoothing term in Proposition~\ref{prop:population-objective} bounds the expected KL error of every smoothing marginal. The population objective therefore penalizes the intermediate filtering errors together with a common upper bound on all smoothing-marginal errors.

\FloatBarrier

\section{Numerical experiments}\label{sec:experiments}
In this section, we present numerical experiments to evaluate the proposed amortized filtering and smoothing flows (AFSF). We first consider a one-dimensional linear advection-diffusion problem, for which exact Gaussian filtering and smoothing distributions are available. We then apply AFSF to Burgers' equation with partial and noisy observations, followed by the single- and two-scale Lorenz systems with nonlinear observations. Finally, we investigate the effect of sharing the observation summary on the accuracy of the filtering and smoothing approximations.

\subsection{Experimental setting}\label{sec:experimental-setting}
For each problem, we use $N_{\mathrm{train}}$ simulated trajectories of length $T_{\mathrm{train}}$ for training and $N_{\mathrm{test}}$ independent trajectories of length $T_{\mathrm{test}}$ for testing, with $T_{\mathrm{test}}\geq T_{\mathrm{train}}$. For nonlinear filtering, we compare AFSF with the flow-based Bayesian filter (FBF)~\cite{wang2025flow}, using the same training and test trajectories.

The summary network uses a one-layer LSTM for the linear problems and the Burgers' equation problem, and a four-layer LSTM for the Lorenz problems. The summary $s_t$ has dimension $3\dim(y_t)$ and is shared by the forward and backward flows in Sections~\ref{linear}--\ref{lorenz96}. Both flows use the conditional KRnet construction introduced in Section~\ref{CNF}, with six affine coupling layers in each flow. Within each coupling layer, a random Fourier feature embedding is followed by a six-layer fully connected network of width $64$ with SiLU activations. The two flows and the summary network are trained jointly by minimizing $\widehat{\mathcal L}$ in \eqref{eq:loss_shared} using Adam with an initial learning rate of $10^{-3}$. The trained parameters remain fixed during inference. AFSF is implemented in PyTorch and trained on an NVIDIA V100 GPU with $32$ GB of memory.

In the linear problems, we report the KL divergence from the exact Gaussian filtering and one-step backward conditional densities to their approximations. Across all examples, the root mean squared error (RMSE) is computed from the posterior sample mean. The squared maximum mean discrepancy ($\mathrm{MMD}^2$) and the componentwise continuous ranked probability score (CRPS) are computed from posterior samples. These three metrics are evaluated against the realized states. The one-step backward conditional distribution is evaluated with $u_{k+1}$ fixed at its realized value on each test trajectory. Smoothing samples conditioned on the full observation sequence $y_{1:T_{\mathrm{test}}}$ are generated by the backward recursion in Algorithm~\ref{alg:predict-smooth}. Each metric is first computed for an individual test trajectory, and the reported values are averaged over the $N_{\mathrm{test}}$ trajectories. The metric definitions and averaging over time are given in Appendix~\ref{appendix:metrics}.

The training and posterior sampling settings used for AFSF are summarized in Table~\ref{tab:afsf-training-settings}. Here, $N_b$ denotes the minibatch size and $N_e$ denotes the number of training epochs. The two flows and the summary network are trained for $N_e$ epochs without early stopping, and the parameter values at the end of training are used for inference. For the linear advection experiment across state dimensions and the advection-diffusion experiment under mesh refinement, $n$ denotes the state dimension and we set $N_e=100n$. The number of posterior samples is denoted by $N_{\mathrm{sample}}$. The filtering and one-step backward columns give the number of samples drawn from the corresponding conditional distribution for each test trajectory at each
assimilation time, whereas the smoothing column gives the number of
joint smoothing trajectories generated by the backward recursion for
each test trajectory.
\begin{table}[ht]
 \caption{Training and posterior sampling settings for AFSF.}
 \label{tab:afsf-training-settings}
 \centering\footnotesize
 \setlength{\tabcolsep}{2.2pt}
 \begin{tabular}{lccccc}
  \toprule
  & & & \multicolumn{3}{c}{$N_{\mathrm{sample}}$} \\
  \cmidrule(lr){4-6}
  \multirow{2}{*}{Problem}
  & \multirow{2}{*}{$N_b$}
  & \multirow{2}{*}{$N_e$}
  & \multirow{2}{*}{Filtering}
  & One-step
  & \multirow{2}{*}{Smoothing} \\
  & & & & backward & \\
  \midrule
  \shortstack[l]{Linear advection:\\state-dimension study}
  & $128$ & $100n$ & $500$ & $500$ & $500$ \\
  \addlinespace
  \shortstack[l]{Advection-diffusion:\\mesh-refinement study}
  & $128$ & $100n$ & $500$ & $500$ & $500$ \\
  \addlinespace
  Burgers' equation
  & $128$ & $20000$ & $500$ & $500$ & $500$ \\
  \addlinespace
  Single-scale Lorenz
  & $128$ & $10000$ & $500$ & $500$ & $500$ \\
  \addlinespace
  Two-scale Lorenz
  & $128$ & $10000$ & $500$ & $500$ & $500$ \\
  \bottomrule
 \end{tabular}
\end{table}
For the FBF comparisons, all FBF models are trained using Adam with
$N_b=64$. For Burgers' equation, FBF uses six layers of width $128$ and $N_e=1000$, whereas for the single- and two-scale Lorenz systems, it uses eight layers of width $128$ and $N_e=5000$. The learning rate is
initialized at $5\times10^{-4}$ and decays geometrically to $5\%$ of its initial value over the course of training. Since FBF is used only for filtering, $N_{\mathrm{sample}}=500$ samples are drawn from the filtering distribution for each test trajectory at each assimilation time.
\subsection{Linear advection-diffusion equation}\label{linear}
We consider the one-dimensional linear advection-diffusion equation
\begin{equation*}
 \partial_tu(t,x)=a\,\partial_xu(t,x)+\kappa\,\partial_{xx}u(t,x),
 \qquad x\in[0,1],
\end{equation*}
with periodic boundary conditions. On the uniform grid $x_j=j\Delta x$, where $\Delta x=1/n$ and $j=0,\ldots,n-1$, let $u_k\in\mathbb R^n$ denote the discrete state at observation time $t_k=k\Delta t$. Let $M_{\delta t}$ be the update matrix for one finite-difference step of size $\delta t$. Over an observation interval $\Delta t=m_n\delta t$, the deterministic propagation matrix is $M=M_{\delta t}^{m_n}$.

The state-space model and initial distribution are
\begin{align*}
 u_k&=Mu_{k-1}+\epsilon_{u,k},\qquad \epsilon_{u,k}\sim\mathcal N(0,Q),\\
 y_k&=Hu_k+\epsilon_{y,k},\qquad \epsilon_{y,k}\sim\mathcal N(0,R),\\
 u_0&\sim\mathcal N(\mu,\Sigma),\qquad
 \mu_j=\sin(2\pi j/n),\qquad \Sigma=\sigma^2I_n.
\end{align*}
For these linear-Gaussian models, the Kalman filter and the Rauch--Tung--Striebel smoother provide exact reference distributions for filtering and smoothing. We use these references to assess approximation accuracy for models of increasing state dimension and to examine inference under mesh refinement of a fixed continuous problem.

\subsubsection{Accuracy across state dimensions}\label{sec:dimension-study}
We first examine how the accuracy of AFSF depends on the state dimension. We take $a=-1$ and $\kappa=0$ and consider a family of finite-dimensional models with state dimensions $n=10,20,30,40,50$. An upwind discretization gives $M_{\delta t}=I-(\delta t/\Delta x)A$, where $A$ is the periodic backward-difference matrix. We set $Q=qI_n$, observe every second grid point through $H$, and use $R=rI_{n/2}$. Across all dimensions, we use
\begin{equation*}
 \Delta t=0.05,\qquad q=0.01,\qquad r=0.1,\qquad \sigma=0.05.
\end{equation*}
For each $n$, we train AFSF on $N_{\mathrm{train}}=2000$ trajectories
of length $T_{\mathrm{train}}=500$ and use $N_{\mathrm{test}}=200$
independent test trajectories of length $T_{\mathrm{test}}=1000$.
The aggregate metrics in Table~\ref{tt1} are computed over the first
$500$ assimilation times, while the full $1000$-step test trajectories,
corresponding to physical time $50$, are used to examine the temporal
evolution of the errors and the extrapolation performance beyond the
training horizon.

The filtering and one-step backward KL divergences increase modestly with $n$, while the filtering and smoothing RMSE remain small and comparatively stable across dimensions, as shown in Table~\ref{tt1}. Incorporating subsequent observations through backward recursion yields lower RMSE, $\mathrm{MMD}^2$, and CRPS than filtering at every tested dimension.

\begin{table}[ht]
    \caption{Filtering, one-step backward, and smoothing results for the linear advection problem at different state dimensions.}\label{tt1}
    \centering\footnotesize
    \resizebox{0.9\textwidth}{!}{%
    \begin{tabular}{lccccccccccc}
        \toprule
        & \multicolumn{4}{c}{Filtering}
        & \multicolumn{4}{c}{One-step backward}
        & \multicolumn{3}{c}{Smoothing} \\
        \cmidrule(lr){2-5}\cmidrule(lr){6-9}\cmidrule(lr){10-12}
        & KL & RMSE & $\mathrm{MMD}^2$ & CRPS & KL & RMSE & $\mathrm{MMD}^2$ & CRPS & RMSE & $\mathrm{MMD}^2$ & CRPS \\
        \midrule
        $n = 10$ & 0.0191 & 0.1457 & 0.0514 & 0.0822 & 0.0146 & 0.0949 & 0.0222 & 0.0536 & 0.1284 & 0.0401 & 0.0725 \\
        $n = 20$ & 0.0443 & 0.1345 & 0.0860 & 0.0759 & 0.0262 & 0.0956 & 0.0446 & 0.0540 & 0.1207 & 0.0704 & 0.0683 \\
        $n = 30$ & 0.0450 & 0.1288 & 0.1164 & 0.0727 & 0.0358 & 0.0960 & 0.0667 & 0.0542 & 0.1172 & 0.0979 & 0.0663 \\
        $n = 40$ & 0.0524 & 0.1252 & 0.1444 & 0.0706 & 0.0434 & 0.0962 & 0.0883 & 0.0543 & 0.1149 & 0.1237 & 0.0650 \\
        $n = 50$ & 0.0597 & 0.1228 & 0.1710 & 0.0693 & 0.0534 & 0.0965 & 0.1097 & 0.0544 & 0.1134 & 0.1484 & 0.0641 \\
        \bottomrule
    \end{tabular}}
\end{table}

At the largest tested dimension, $n=50$, we further examine the time dependence of the approximation errors. The filtering and one-step backward KL divergences in Figure~\ref{oneadkl} decrease initially and then increase gradually over the test interval.

\begin{figure}[ht]
 \centering
 \includegraphics[width=0.4\textwidth]{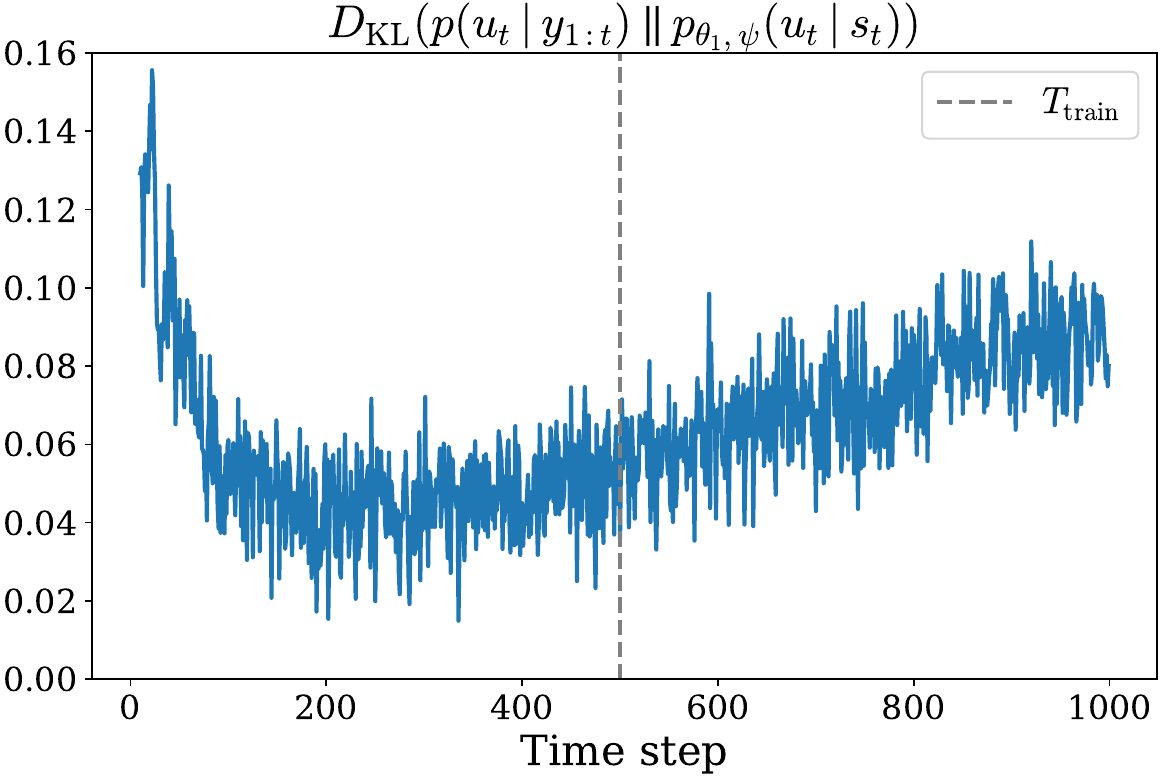}
 \includegraphics[width=0.4\textwidth]{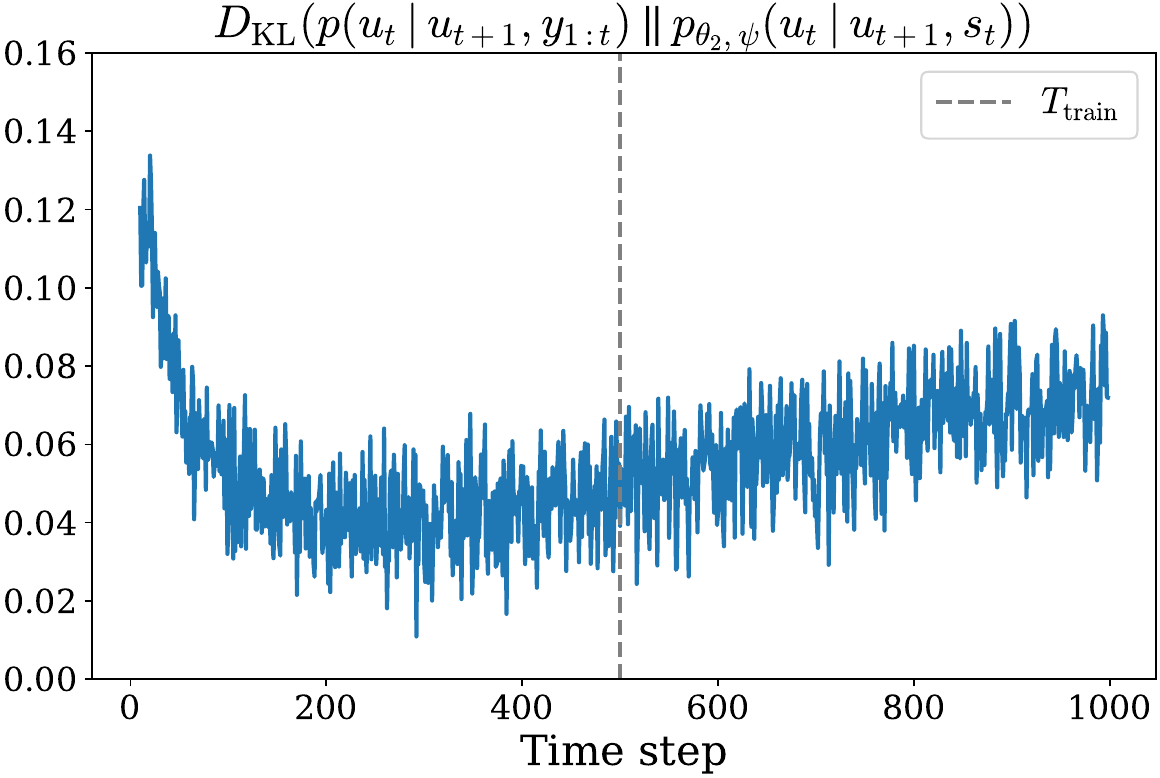}
 \caption{Filtering (left) and one-step backward (right) KL divergences at $n=50$. The exact Gaussian conditionals are used as the first arguments. The vertical dashed line marks the end of the training interval.}
 \label{oneadkl}
\end{figure}

\begin{figure}[ht]
 \centering
 \includegraphics[width=0.88\textwidth]{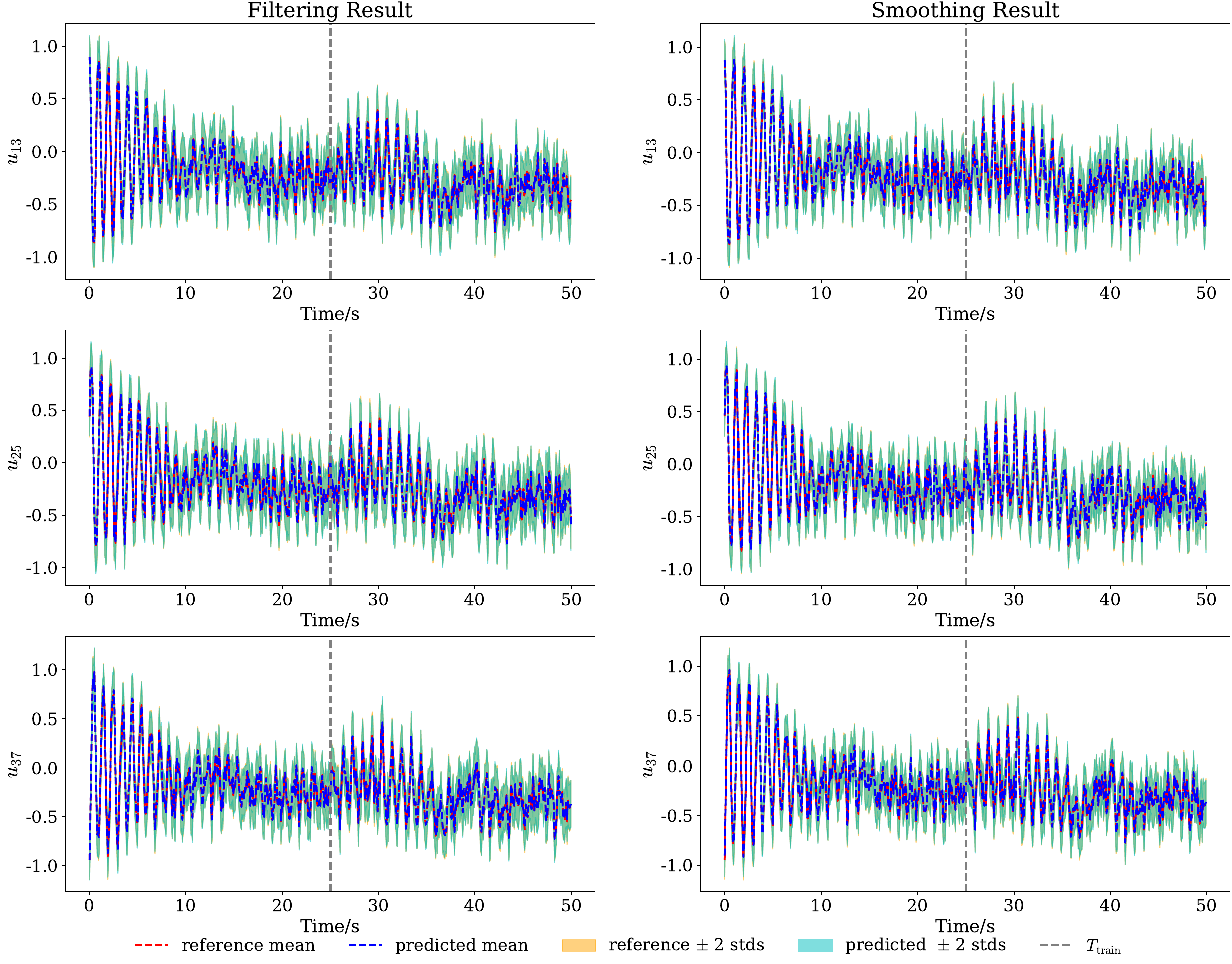}
 \caption{Exact Gaussian reference mean, AFSF posterior mean, and pointwise two-standard-deviation bands for filtering (left) and smoothing (right) at three representative state components with $n=50$. The vertical line marks the end of the training interval.}
 \label{fig1}
\end{figure}
Trained on trajectories of length $500$, AFSF yields filtering and smoothing means and marginal standard deviations that remain close to their exact Gaussian references over the test interval of $1000$ steps, as shown in Figure~\ref{fig1}. This agreement suggests that the recurrent summary network captures temporal dependence in the observations, enabling the conditional flows to approximate the evolving posterior distributions at later times.

\subsubsection{Accuracy under mesh refinement}\label{sec:mesh-refinement}\label{sec:gaussian-smoothing}
We next study filtering and smoothing under spatial refinement of the advection-diffusion equation with $a=1$ and $\kappa=0.01$. The discrete observations approximate averages of the continuous field over eight fixed equal subintervals of $[0,1]$, while the initial and process perturbations are scaled with the mesh. On each grid, the corresponding groups $\mathcal I_i$ contain $n/8$ state components, and the observations are
\begin{equation}\label{case2observation}
 y_{k,i}=\frac1{|\mathcal I_i|}\sum_{j\in\mathcal I_i}u_{k,j}+\epsilon_{y,k,i},
 \qquad \epsilon_{y,k}\sim\mathcal N(0,rI_8).
\end{equation}
Gaussian noise with covariance $Q_{\delta t}$ is added at each finite-difference time step. Propagating these noise increments to the next observation time gives the transition covariance $Q$,
\begin{equation*}
 Q_{\delta t}=\frac{\delta t}{n}I_n,
 \qquad
 Q=\sum_{i=0}^{m_n-1}M_{\delta t}^{i}Q_{\delta t}(M_{\delta t}^{i})^\top,
\end{equation*}
where $M_{\delta t}$ is obtained from a Lax--Wendroff-type discretization. We set $\Delta t=0.01$, $r=0.01$, and $\sigma=0.05/\sqrt{n}$, and test $n=16,32,48,64$. For each resolution, we train AFSF on $N_{\mathrm{train}}=2000$ trajectories of length $T_{\mathrm{train}}=500$ and use $N_{\mathrm{test}}=200$ independent test trajectories of length $T_{\mathrm{test}}=1000$, corresponding to physical time $10$. The aggregate metrics in Table~\ref{tt2} are computed over the first $500$ assimilation times, while the full test trajectories are used to examine the temporal evolution of the errors and the extrapolation performance beyond the training horizon.

The filtering and one-step backward KL divergences in Table~\ref{tt2} remain comparable as the mesh is refined, indicating that finer spatial discretizations do not lead to a systematic increase in the conditional density approximation errors. The filtering and smoothing RMSE, $\mathrm{MMD}^2$, and CRPS decrease with mesh refinement. Smoothing further reduces $\mathrm{MMD}^2$ and CRPS at every resolution, with comparable or lower RMSE.

\begin{table}[ht]
	\caption{Filtering, one-step backward, and smoothing results for the advection-diffusion problem under mesh refinement.}\label{tt2}
	\centering\footnotesize
	\resizebox{1.0\textwidth}{!}{%
	\begin{tabular}{lccccccccccc}
		\toprule
		& \multicolumn{4}{c}{Filtering}
		& \multicolumn{4}{c}{One-step backward}
		& \multicolumn{3}{c}{Smoothing} \\
		\cmidrule(lr){2-5}\cmidrule(lr){6-9}\cmidrule(lr){10-12}
		& KL & RMSE & $\mathrm{MMD}^2$ & CRPS & KL & RMSE & $\mathrm{MMD}^2$ & CRPS & RMSE & $\mathrm{MMD}^2$ & CRPS \\
		\midrule
		$n = 16$ & 0.1355 & 0.0558 & 0.0123 & 0.0314 & 0.0499 & 0.0236 & 0.0022 & 0.0133 & 0.0548 & 0.0102 & 0.0281 \\
		$n = 32$ & 0.1363 & 0.0360 & 0.0103 & 0.0203 & 0.0653 & 0.0162 & 0.0021 & 0.0091 & 0.0360 & 0.0082 & 0.0182 \\
		$n = 48$ & 0.1135 & 0.0267 & 0.0085 & 0.0150 & 0.0587 & 0.0121 & 0.0017 & 0.0068 & 0.0237 & 0.0067 & 0.0134 \\
		$n = 64$ & 0.1268 & 0.0214 & 0.0073 & 0.0121 & 0.0592 & 0.0095 & 0.0014 & 0.0054 & 0.0190 & 0.0058 & 0.0107 \\
		\bottomrule
	\end{tabular}}
\end{table}


At $n=48$, the filtering and smoothing means in Figure~\ref{linearcase2} reproduce the oscillations and gradual decay of the Gaussian reference, while the marginal standard deviations also remain in close agreement.

\begin{figure}[ht]
 \centering
 \includegraphics[width=.88\textwidth]{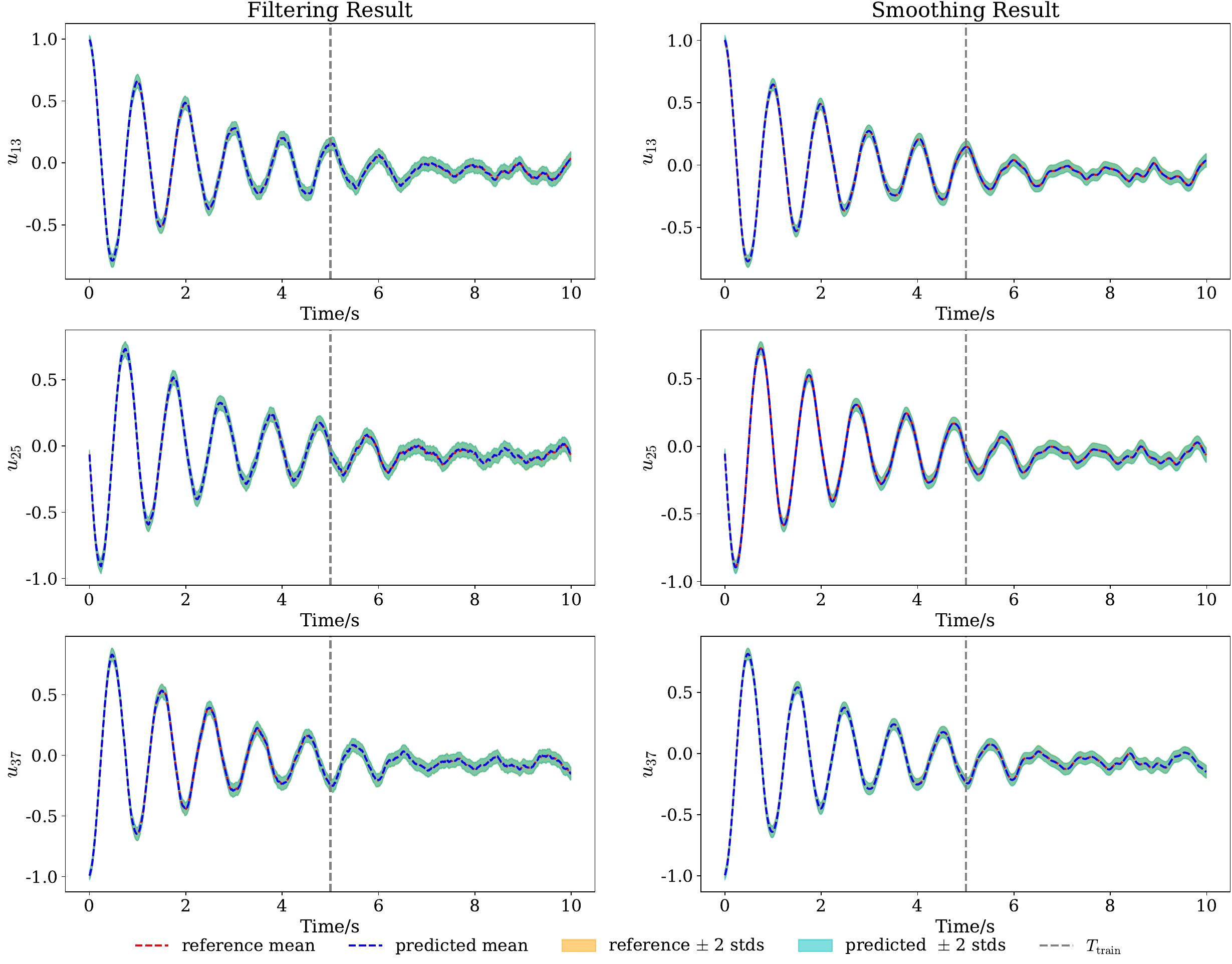}
 \caption{Exact Gaussian reference mean, AFSF posterior mean, and pointwise two-standard-deviation bands for filtering (left) and smoothing (right) at three representative state components with $n=48$. The vertical line marks the end of the training interval.}
 \label{linearcase2}
\end{figure}

The rapid initial decrease in the KL divergences in Figure~\ref{case2kl} may reflect the accumulation of observation information in the LSTM states, which are initialized at zero. Both divergences remain small over the test interval, with only modest growth beyond the training interval.

\begin{figure}[ht]
 \centering
 \includegraphics[width=0.4\textwidth]{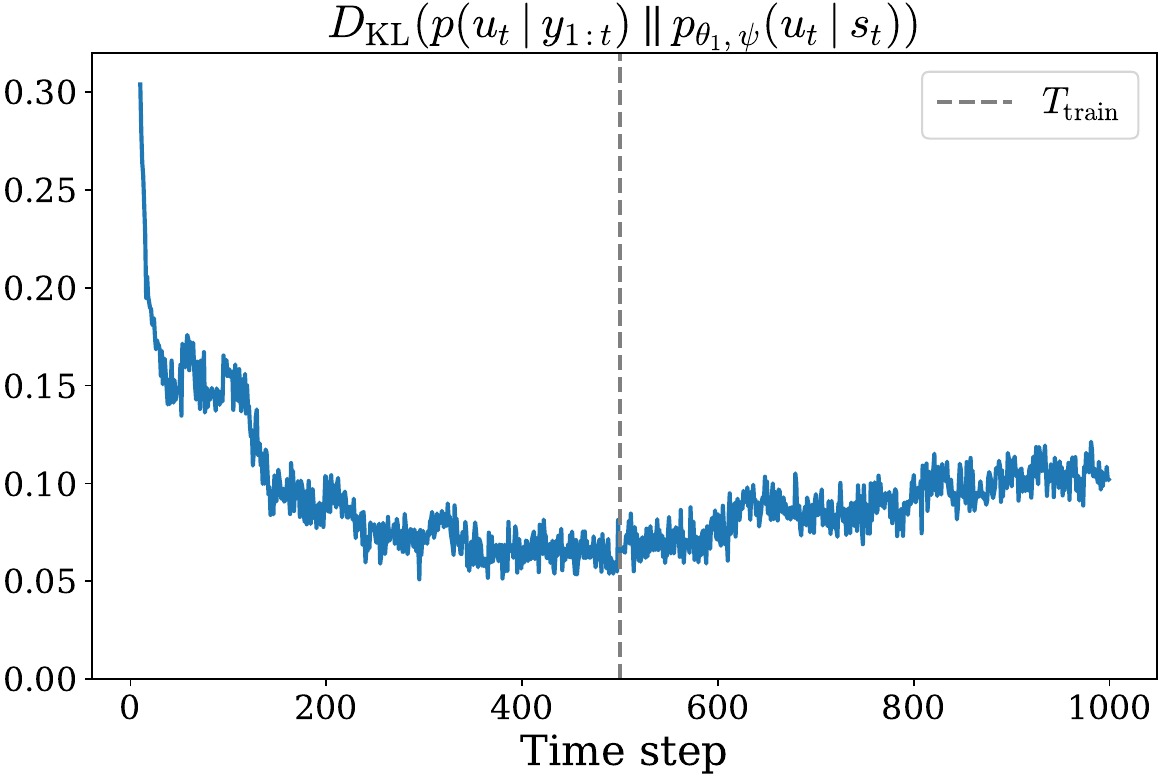}
 \includegraphics[width=0.4\textwidth]{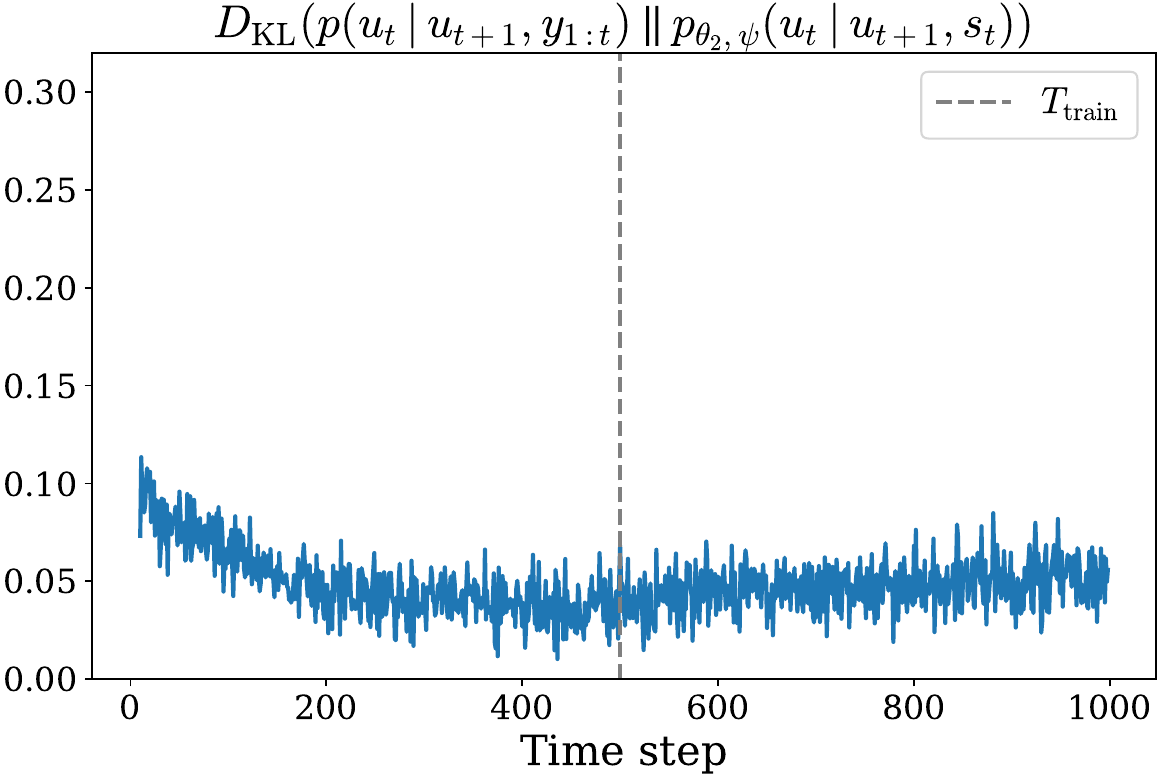}
 \caption{Filtering (left) and one-step backward (right) KL divergences for the mesh-refinement experiment at $n=48$. The exact Gaussian conditionals are used as the first arguments. The vertical dashed line marks the end of the training interval.}
 \label{case2kl}
\end{figure}

\subsection{Burgers' equation}\label{burgers}
We next consider reconstruction of a nonlinear random field from incomplete and noisy observations. We examine how observation noise affects reconstruction and assess the benefit of incorporating subsequent observations through smoothing. The state evolves according to the stochastic Burgers' equation
\begin{equation}\label{burgersequation}
\left\{
\begin{aligned}
 \mathrm{d}u+(u\,\partial_xu-\nu\,\partial_{xx}u)\,\mathrm{d}t&=\sigma\,\mathrm{d}W(t), &&x\in[-1,1],\ t\in[0,1],\\
 u(0,x)&=-\sin(\pi x), &&u(t,-1)=u(t,1)=0,
\end{aligned}
\right.
\end{equation}
with $\nu=0.05$ and $\sigma=1$.

Trajectories are generated by a hybrid finite-difference solver using 201 time levels and 50 spatial locations, with $\Delta t=0.005$. The state consists of the solution at these 50 locations, while 25 alternating locations are observed with noise $\epsilon_{y,k}\sim\mathcal N(0,r^2I_{25})$. We consider observation-noise variances $r^2=0.01,0.04,0.09,0.16,0.25$. We use $N_{\mathrm{train}}=3000$ training trajectories and $N_{\mathrm{test}}=200$ independent test trajectories, with $T_{\mathrm{train}}=T_{\mathrm{test}}=200$.

The filtering errors in Table~\ref{burgerst1} increase with the observation-noise variance, reflecting the reduced information available from the observations. AFSF and FBF perform similarly at the lowest noise level, while AFSF yields modestly smaller values of all three metrics at the higher noise levels.

\begin{table}[ht]
		\caption{Filtering metrics for Burgers' equation with $\nu=0.05$ at different observation-noise variances.}\label{burgerst1}
		\centering\footnotesize
		\begin{tabular}{lcccccc}
			\toprule
			& \multicolumn{2}{c}{RMSE}
			& \multicolumn{2}{c}{$\mathrm{MMD}^2$}
			& \multicolumn{2}{c}{CRPS}\\
			\cmidrule(lr){2-3}\cmidrule(lr){4-5}
			\cmidrule(lr){6-7}
			& AFSF & FBF & AFSF & FBF & AFSF & FBF \\
			\midrule
			$r^2 = 0.01$ & 0.0751&0.0752&0.0683&0.0681&0.0411&0.0410\\
			$r^2 = 0.04$ & 0.0898&0.0917&0.0961&0.0994&0.0497&0.0506\\
			$r^2 = 0.09$ & 0.1003&0.1015&0.1184&0.1203&0.0556&0.0561\\
			$r^2 = 0.16$ & 0.1084&0.1113&0.1369&0.1426&0.0601&0.0615\\
			$r^2 = 0.25$ & 0.1149&0.1179&0.1523&0.1584&0.0637&0.0650\\
			\bottomrule
		\end{tabular}
	\end{table}

The one-step backward metrics in Table~\ref{burgerst2} vary less with observation noise than the filtering metrics. This is consistent with the additional information supplied by the realized next state on which the backward conditional is evaluated.

Incorporating subsequent observations through backward recursion lowers $\mathrm{MMD}^2$ and CRPS relative to filtering at every noise level, while the smoothing RMSE is slightly higher than that of filtering at two noise levels. At the largest observation-noise variance, $r^2=0.25$, both the filtering and smoothing means recover the main spatial structure of the field and its evolution over time, as shown in Figure~\ref{burgers0.05st}.

\begin{table}[ht]
		\caption{One-step backward and smoothing metrics for Burgers' equation with $\nu=0.05$ at different observation-noise variances.}\label{burgerst2}
		\centering\footnotesize
		\begin{tabular}{lcccccc}
			\toprule
			& \multicolumn{3}{c}{One-step backward}
			& \multicolumn{3}{c}{Smoothing} \\
			\cmidrule(lr){2-4}\cmidrule(lr){5-7}
			& RMSE & $\mathrm{MMD}^2$ & CRPS & RMSE & $\mathrm{MMD}^2$ & CRPS \\
			\midrule
			$r^2=0.01$ &0.0525&0.0340&0.0290&0.0710&0.0570&0.0373\\
			$r^2=0.04$ &0.0562&0.0387&0.0311&0.0879&0.0795&0.0458\\
			$r^2=0.09$ &0.0574&0.0405&0.0318&0.1056&0.0950&0.0502\\
			$r^2=0.16$ &0.0581&0.0414&0.0322&0.0981&0.1088&0.0531\\
			$r^2=0.25$ &0.0584&0.0418&0.0324&0.1185&0.1236&0.0581\\
			\bottomrule
		\end{tabular}
	\end{table}

\begin{figure}[!htbp]
 \centering
 \includegraphics[width=.90\textwidth]{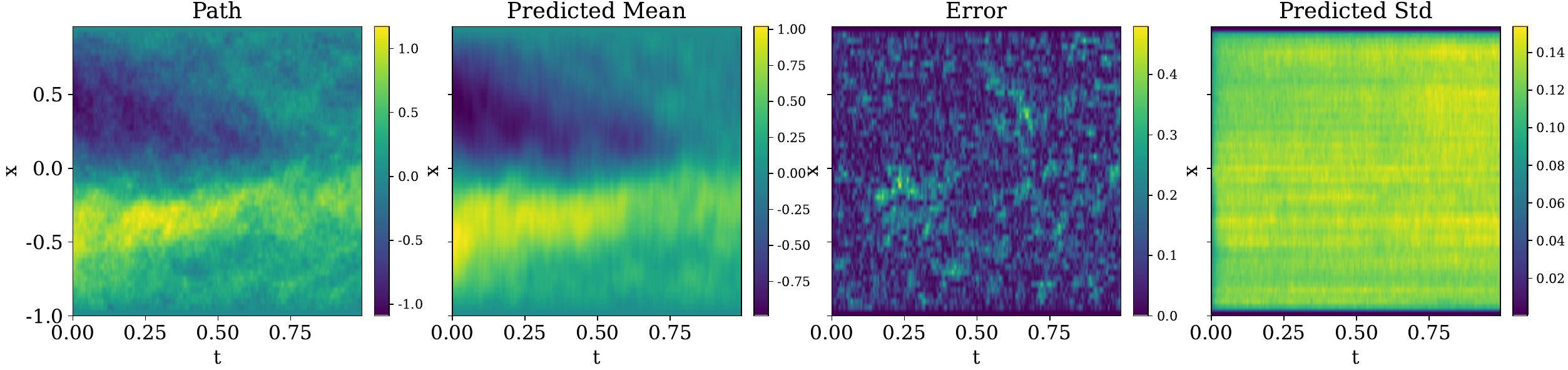}\par\smallskip
 \includegraphics[width=.90\textwidth]{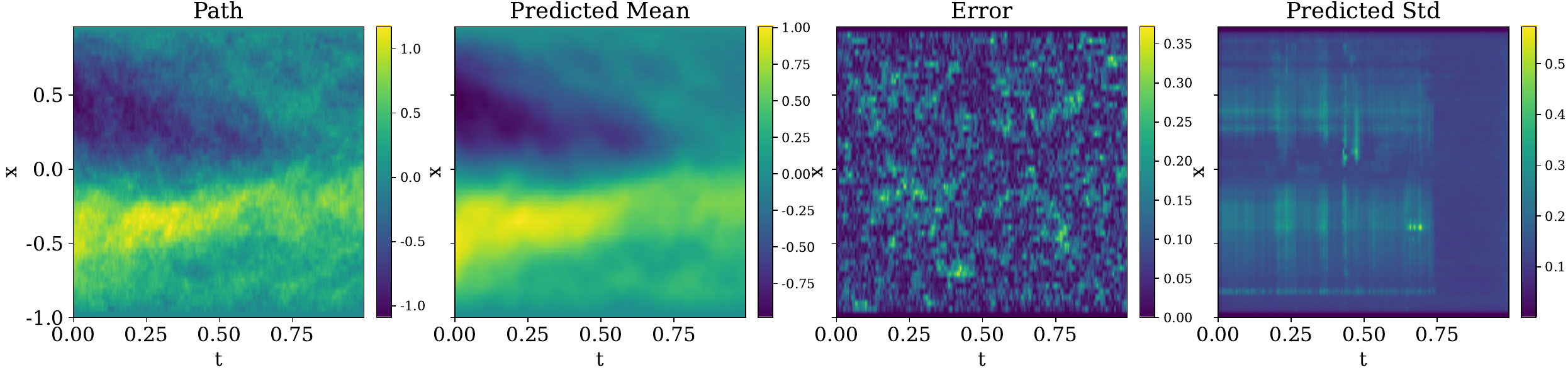}
 \caption{Filtering (top) and smoothing (bottom) for Burgers' equation with $\nu=0.05$ and $r^2=0.25$. The columns show the realized field, posterior mean, absolute error of the posterior mean, and marginal standard deviation.}
 \label{burgers0.05st}
\end{figure}
Appendix~\ref{appendix:burgers-trajectories} presents the posterior means and uncertainty bands over time at three spatial locations, together with filtering results for the lower viscosity $\nu=0.01$. At this lower viscosity, filtering errors are larger, but AFSF retains a modest advantage over FBF.

\FloatBarrier

\subsection{Lorenz systems}\label{lorenz96}
We consider the stochastic two-scale Lorenz system~\cite{wilks2005effects}, with $K$ slow variables and $KJ$ fast variables,
\begin{equation}\label{lorenz96evo}
\resizebox{0.91\linewidth}{!}{$\displaystyle
\left\{
\begin{aligned}
 \mathrm{d}u_i(t)={}&\left[-u_{i-1}(t)\bigl(u_{i-2}(t)-u_{i+1}(t)\bigr)-u_i(t)+F
 -\frac{hc}{b}\sum_{j=(i-1)J+1}^{iJ}v_j(t)\right]\mathrm{d}t+\sigma_u\,\mathrm{d}W_{u,i}(t),\\
 \mathrm{d}v_j(t)={}&\left[-cbv_{j+1}(t)\bigl(v_{j+2}(t)-v_{j-1}(t)\bigr)-cv_j(t)
 +\frac{hc}{b}u_{\lfloor(j-1)/J\rfloor+1}(t)\right]\mathrm{d}t+\sigma_v\,\mathrm{d}W_{v,j}(t),
\end{aligned}
\right.
$}
\end{equation}
for $i=1,\ldots,K$ and $j=1,\ldots,KJ$, with cyclic indices. Each slow variable is coupled to $J$ fast variables. The parameters $c$ and $h$ determine the time-scale ratio and coupling strength, respectively, and $F$ is the forcing.
We first consider the single-scale system with $c=0$, and then retain the coupling at $c=4$ to study inference of the partially observed slow variables.

\subsubsection{Single-scale system}\label{sec:single-scale}
The single-scale system consists only of the $K$ variables $u_i$, whose dynamics are given by the first equation in \eqref{lorenz96evo} with $c=0$. All $K$ variables are observed through
\begin{equation*}
 y_{k,i}=u_{k,i}^3+\epsilon_{y,k,i},\qquad
 \epsilon_{y,k,i}\sim\mathcal N(0,1),\qquad i=1,\ldots,K.
\end{equation*}
We set $\sigma_u=1$, $F=8$, $\Delta t=0.05$, and $u_i(0)=\sin(2\pi i/K)$, and consider $K=10,20,30,40,50$. For each dimension, we train AFSF on $N_{\mathrm{train}}=2000$ trajectories
of length $T_{\mathrm{train}}=500$ and use $N_{\mathrm{test}}=200$ independent test trajectories of length $T_{\mathrm{test}}=1000$. The aggregate metrics in Table~\ref{lorenz96vs} are computed over the first $500$ assimilation times, while the full $1000$-step test trajectories, corresponding to physical time $50$, are used to examine the temporal
evolution of the errors and the extrapolation performance beyond the
training horizon.

AFSF maintains a consistent advantage over FBF as the state dimension increases; see Table~\ref{lorenz96vs}. The relative reduction in filtering RMSE is approximately $16$--$20\%$ across the tested dimensions, accompanied by lower $\mathrm{MMD}^2$ and CRPS. AFSF also maintains lower filtering errors than FBF beyond the training interval at $K=50$, as shown in Figure~\ref{lorenz96filterrmse}. The metrics fluctuate around stable levels after the initial transient.
\begin{table}[ht]
		\caption{Filtering metrics for the single-scale Lorenz system at different state dimensions.}\label{lorenz96vs}
		\centering\footnotesize
		\begin{tabular}{lcccccc}
			\toprule
			& \multicolumn{2}{c}{RMSE}
			& \multicolumn{2}{c}{$\mathrm{MMD}^2$}
			& \multicolumn{2}{c}{CRPS}                                                                                                                                              \\
			\cmidrule(lr){2-3}\cmidrule(lr){4-5}
			\cmidrule(lr){6-7}
			& AFSF & FBF & AFSF & FBF & AFSF & FBF \\
			\midrule
			$K = 10$ &0.1632&0.2044&0.0647&0.0967&0.0738&0.1015\\
			$K = 20$ &0.1945&0.2439&0.1703&0.2530&0.0902&0.1272\\
			$K = 30$ &0.2081&0.2604&0.2675&0.3987&0.0976&0.1393\\
			$K = 40$ &0.2255&0.2742&0.3844&0.5295&0.1131&0.1489\\
			$K = 50$ &0.2605&0.3106&0.5508&0.6918&0.1366&0.1668\\
			\bottomrule
		\end{tabular}
	\end{table}
\begin{figure}[ht]
 \centering
 \includegraphics[width=.88\textwidth]{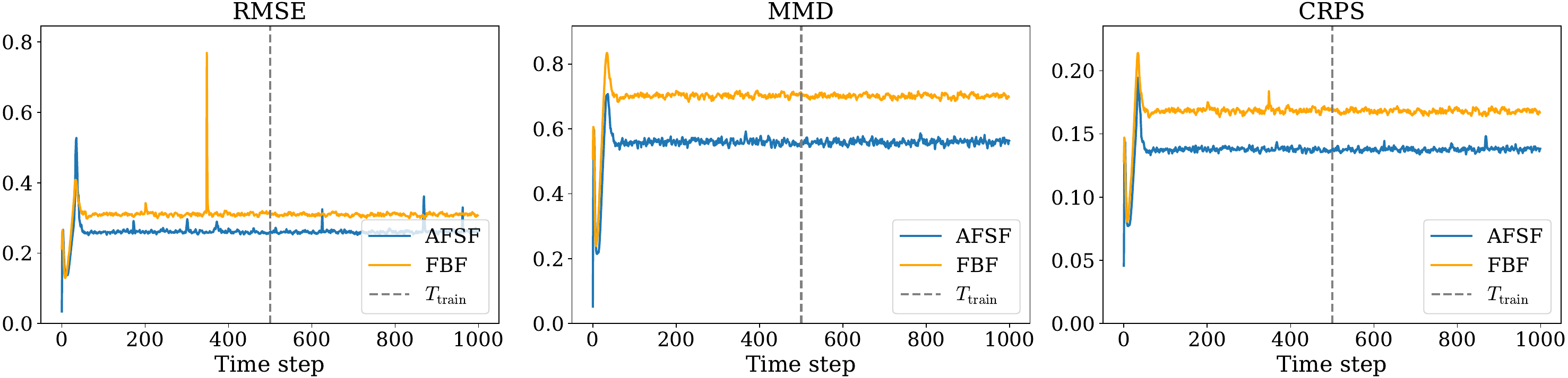}
 \caption{Time evolution of filtering RMSE, $\mathrm{MMD}^2$, and CRPS for AFSF and FBF in the single-scale Lorenz system at $K=50$. The vertical dashed line marks the end of the training interval.}
 \label{lorenz96filterrmse}
\end{figure}
Incorporating subsequent observations through backward recursion improves all three metrics relative to filtering for $K\leq40$; see Tables~\ref{lorenz96vs} and~\ref{lorenz96fs}. At $K=50$, smoothing yields larger RMSE and CRPS than filtering, although $\mathrm{MMD}^2$ remains lower. From $K=40$ to $K=50$, the errors increase more sharply for smoothing than for the one-step backward evaluation, which conditions on the realized next state. Errors in the approximate backward kernels can propagate through successive steps of the smoothing recursion.

In an additional computation with $N_{\mathrm{train}}=3000$ at $K=50$, the smoothing RMSE, $\mathrm{MMD}^2$, and CRPS decrease to $0.2063$, $0.4034$, and $0.1085$, respectively. With this larger training set, the filtering and smoothing means in Figure~\ref{lorenz96fusion} follow the rapid temporal variations of the realized states throughout the test interval.

\begin{table}[ht]
		\caption{One-step backward and smoothing metrics for the single-scale Lorenz system.}\label{lorenz96fs}
		\centering\footnotesize
		\begin{tabular}{lcccccc}
			\toprule
			& \multicolumn{3}{c}{One-step backward}
			& \multicolumn{3}{c}{Smoothing} \\
			\cmidrule(lr){2-4}\cmidrule(lr){5-7}
			& RMSE & $\mathrm{MMD}^2$ & CRPS & RMSE & $\mathrm{MMD}^2$ & CRPS \\
			\midrule
			$K = 10$ &0.1105&0.0301&0.0527&0.1298&0.0428&0.0598 \\
			$K = 20$ &0.1240&0.0738&0.0630&0.1525&0.1108&0.0742\\
			$K = 30$ &0.1347&0.1253&0.0697&0.1680&0.1896&0.0838\\
			$K = 40$ &0.1499&0.1979&0.0810&0.1881&0.2950&0.0991 \\
			$K = 50$ &0.1695&0.2966&0.0935&0.5423&0.4594&0.1844\\
			\bottomrule
		\end{tabular}
	\end{table}

\begin{figure}[ht]
 \centering
 \includegraphics[width=.88\textwidth]{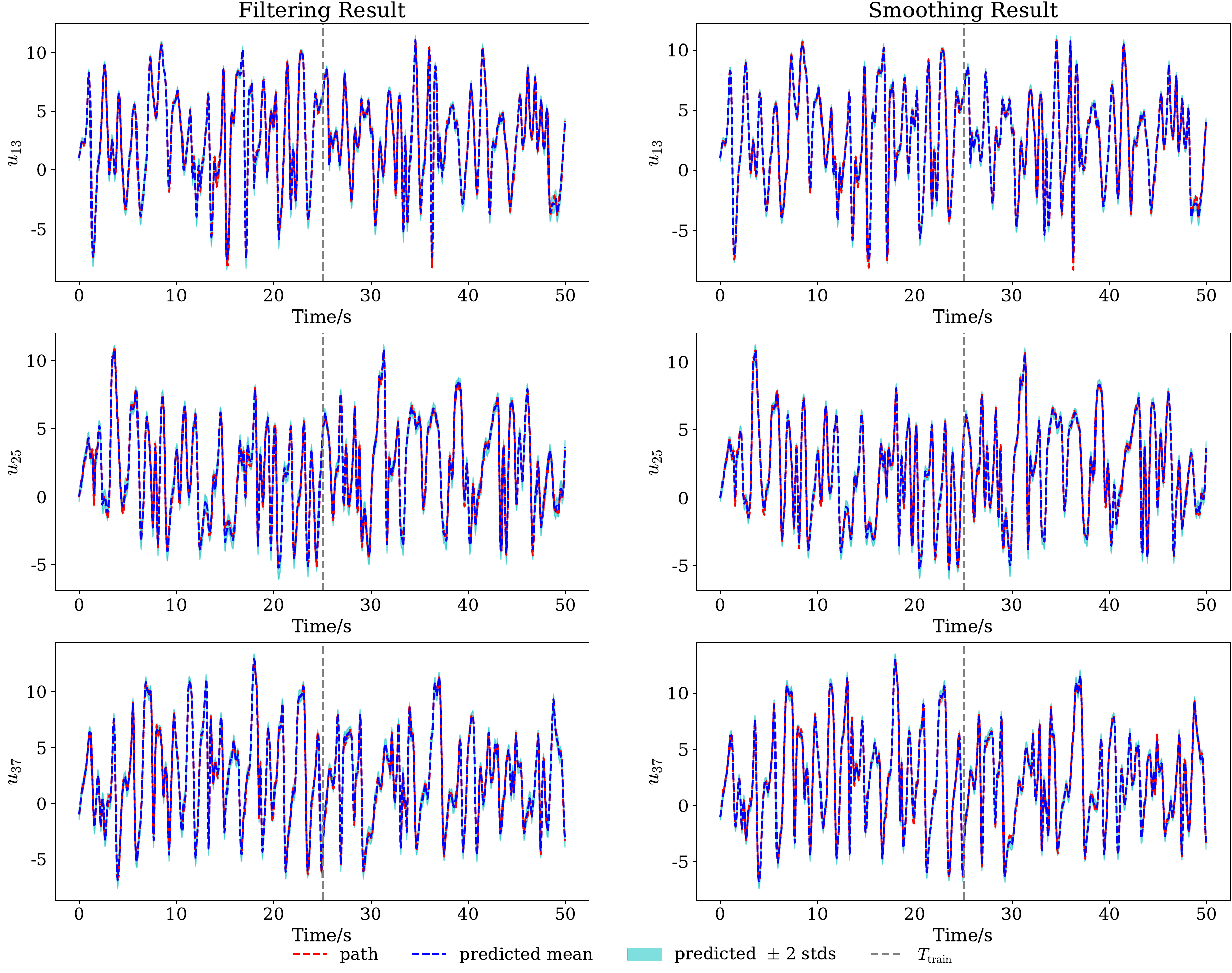}
 \caption{Realized path, posterior mean, and pointwise two-standard-deviation bands for filtering (left) and smoothing (right) at three representative state components of the single-scale Lorenz system with $K=50$ and $N_{\mathrm{train}}=3000$. The vertical line marks the end of the training interval.}
 \label{lorenz96fusion}
\end{figure}

\subsubsection{Two-scale system}\label{sec:two-scale}
We next consider the full two-scale system with $c=4$, in which the slow variables remain coupled to the fast variables. We investigate whether the slow variables can be reconstructed from partial and noisy observations, with the fast variables entirely unobserved. The observations are taken at the odd-indexed slow variables according to
\begin{equation*}
 y_{k,i}=u_{k,2i-1}^3+\epsilon_{y,k,i},\qquad
 \epsilon_{y,k,i}\sim\mathcal N(0,1),\qquad i=1,\ldots,K/2.
\end{equation*}
We use $K=16$, $J=32$, $h=1$, $b=10$, $\sigma_u=1$, $\sigma_v=0.1$, and $\Delta t=0.01$. The initial conditions are
\begin{equation*}
 u_i(0)=F+0.1\,\xi_{u,i},\qquad
 v_j(0)=0.01\,\xi_{v,j},\qquad
 \xi_{u,i},\xi_{v,j}\stackrel{\mathrm{i.i.d.}}{\sim}\mathcal N(0,1).
\end{equation*}
The forcing is varied over $F=5,8,16$. Training and test trajectories are generated from the full coupled system \eqref{lorenz96evo}, and the conditional flows are trained on the slow states and their observations. For each forcing value, we train AFSF on $N_{\mathrm{train}}=2000$ trajectories
of length $T_{\mathrm{train}}=500$ and use $N_{\mathrm{test}}=200$
independent test trajectories of length $T_{\mathrm{test}}=1000$.
The aggregate metrics in Tables~\ref{twoscalevs} and~\ref{twoscalermse}
are computed over the first $500$ assimilation times, while the full
$1000$-step test trajectories, corresponding to physical time $10$, are
used to assess performance beyond the training horizon. The metrics are computed for the full slow state.

After the fast variables are marginalized, the slow process need not remain Markov. Consequently, the exact backward conditional for $u_k$ may depend on subsequent slow states beyond $u_{k+1}$. We nevertheless retain the one-step backward approximation in \eqref{eq:learned-smoothing-path}, conditioning each backward kernel only on the next slow state $u_{k+1}$ and the observation summary $s_k$.

AFSF yields lower filtering RMSE, $\mathrm{MMD}^2$, and CRPS than FBF at all three forcing values, as reported in Table~\ref{twoscalevs}. The AFSF errors vary moderately across the three forcing values, while the FBF errors increase substantially at $F=16$.

With this approximation, backward recursion reduces RMSE and $\mathrm{MMD}^2$ relative to filtering at every forcing value, while CRPS is lower at $F=5$ and $8$ and slightly higher at $F=16$; see Tables~\ref{twoscalevs} and~\ref{twoscalermse}. At $F=5$, the filtering and smoothing means in Figure~\ref{twoscalefusion} follow the temporal variations of the observed slow variables beyond the training interval. The wider uncertainty bands when these variables are near zero are consistent with the reduced local sensitivity of the cubic observation function.

\begin{table}[ht]
	\caption{Filtering metrics for the two-scale Lorenz system at different forcing values.}
	\label{twoscalevs}
	\centering\footnotesize
	\begin{tabular}{lcccccc}
		\toprule
		& \multicolumn{2}{c}{RMSE}
		& \multicolumn{2}{c}{$\mathrm{MMD}^2$}
		& \multicolumn{2}{c}{CRPS} \\
		\cmidrule(lr){2-3}
		\cmidrule(lr){4-5}
		\cmidrule(lr){6-7}
		& AFSF & FBF & AFSF & FBF & AFSF & FBF \\
		\midrule
		$F=5$  & 0.4634 & 0.5356 & 0.5464 & 0.6065 & 0.2169 & 0.2450 \\
		$F=8$  & 0.4565 & 0.7756 & 0.5157 & 0.7799 & 0.2105 & 0.3472 \\
		$F=16$ & 0.4872 & 2.4911 & 0.5030 & 1.0136 & 0.2258 & 1.0967 \\
		\bottomrule
	\end{tabular}
\end{table}
\begin{table}[ht]
	\caption{One-step backward and smoothing metrics for the two-scale Lorenz system.}
	\label{twoscalermse}
	\centering\footnotesize
	\begin{tabular}{lcccccc}
		\toprule
		& \multicolumn{3}{c}{One-step backward}
		& \multicolumn{3}{c}{Smoothing} \\
		\cmidrule(lr){2-4}
		\cmidrule(lr){5-7}
		& RMSE & $\mathrm{MMD}^2$ & CRPS
		& RMSE & $\mathrm{MMD}^2$ & CRPS \\
		\midrule
		$F=5$  & 0.0859 & 0.0290 & 0.0453 & 0.4232 & 0.4600 & 0.1909 \\
		$F=8$  & 0.0884 & 0.0308 & 0.0471 & 0.4121 & 0.4286 & 0.1839 \\
		$F=16$ & 0.1024 & 0.0407 & 0.0564 & 0.5141 & 0.4454 & 0.2249 \\
		\bottomrule
	\end{tabular}
\end{table}
\begin{figure}[!htbp]
 \centering
 \includegraphics[width=.88\textwidth]{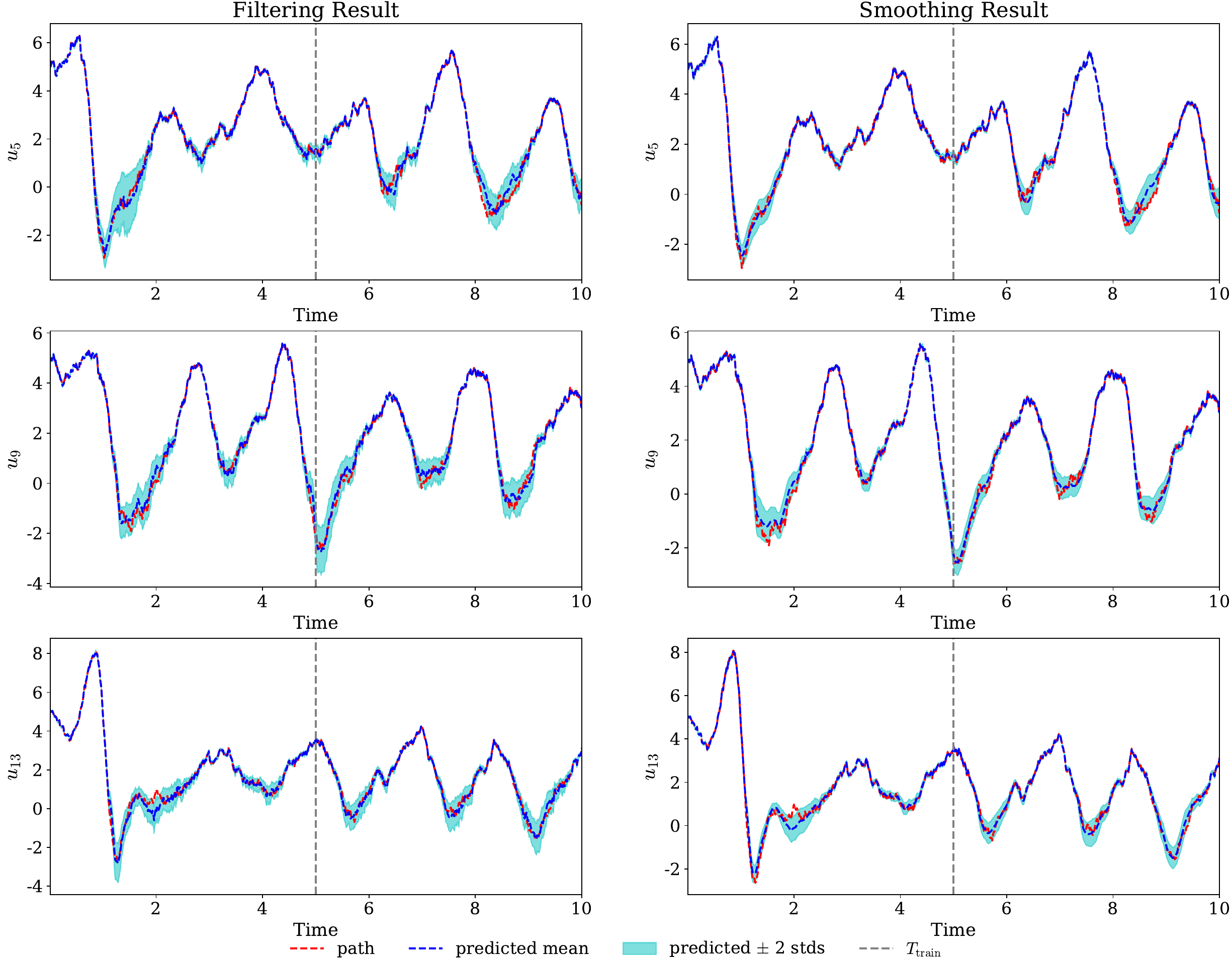}
 \caption{Realized slow-state path, posterior mean, and pointwise two-standard-deviation bands for filtering (left) and smoothing based on the one-step backward approximation (right) for three observed slow variables at $F=5$. The vertical lines mark the end of the training interval.}
 \label{twoscalefusion}
\end{figure}

\subsection{Effect of sharing the observation summary}\label{sec:summary-comparison}
We compare shared and separate observation summaries to examine their effects on filtering, one-step backward inference, and smoothing. In the separate formulation, the filtering and backward flows each have their own recurrent network, with the layer structure and summary dimension specified in Section 4.1. The simulated trajectories, flow architectures, conditional training objectives, and optimization settings are otherwise unchanged. Thus, the separate formulation uses two recurrent networks, each of the same size as the single network in the shared formulation.

The results in Table~\ref{tab:shared-summaries} show that the filtering
and one-step backward results are broadly comparable with shared and
separate summaries. The filtering scores are nearly unchanged, while the
one-step backward errors are somewhat larger with separate summaries in
the Lorenz systems. A much larger difference appears in smoothing. The
two choices give nearly identical smoothing results for the linear
problem, whereas separate summaries lead to very large smoothing errors
for the strongly nonlinear Lorenz systems.

Unlike the one-step backward evaluation, which conditions on the realized
next state, smoothing repeatedly applies the backward kernels to states
generated by preceding backward steps. Approximation errors can therefore
propagate and accumulate through the backward recursion. The results suggest that sharing the observation summary improves the stability of recursive smoothing in the nonlinear examples considered. Sharing also reduces the number of recurrent networks from two to one, yielding a more compact model for joint training.
\begin{table}[ht]
	\caption{Filtering, one-step backward, and smoothing scores with shared and
	separate observation summaries at $T=500$. The two-scale smoothing scores
	refer to the first-order closure.}
	\label{tab:shared-summaries}
	\centering\footnotesize
	\begin{tabular}{lcccccc}
		\toprule
		& \multicolumn{2}{c}{Filtering}
		& \multicolumn{2}{c}{One-step backward}
		& \multicolumn{2}{c}{Smoothing} \\
		\cmidrule(lr){2-3}
		\cmidrule(lr){4-5}
		\cmidrule(lr){6-7}
		Summary & RMSE & CRPS & RMSE & CRPS & RMSE & CRPS \\
		\midrule
        \multicolumn{7}{l}{Linear advection, $n=50$} \\
Shared   & 0.1228 & 0.0693 & 0.0965 & 0.0544 & 0.1134 & 0.0641 \\
Separate & 0.1232 & 0.0696 & 0.0967 & 0.0547 & 0.1139 & 0.0644 \\
\addlinespace
		\multicolumn{7}{l}{Single-scale Lorenz, $K=50$} \\
		Shared   & 0.1945 & 0.0902 & 0.1240 & 0.0630 & 0.1525  & 0.0742 \\
		Separate & 0.1958 & 0.0916 & 0.1596 & 0.0816 & 60.3744 & 33.2893 \\
		\addlinespace
		\multicolumn{7}{l}{Two-scale Lorenz, $F=16$} \\
		Shared   & 0.4872 & 0.2258 & 0.1024 & 0.0564 & 0.5141  & 0.2249 \\
		Separate & 0.5120 & 0.2316 & 0.1161 & 0.0621 & 55.4369 & 24.6898 \\
		\bottomrule
	\end{tabular}
\end{table}
\FloatBarrier

\section{Conclusion}\label{sec:conclusion}
In this work, we developed an amortized method for Bayesian filtering and smoothing based on conditional normalizing flows. The forward flow approximates the filtering distribution and, together with the backward flow, defines an approximate joint smoothing distribution. The two flows share a recurrent summary network. The parameters of both flows and the shared summary network are estimated jointly from simulated trajectories and remain fixed during inference for different observation sequences and assimilation times. Numerical results on advection-diffusion, Burgers, and Lorenz systems illustrate the accuracy of the filtering and smoothing approximations. Future work will consider extensions to higher-dimensional problems and dynamical systems with time delays.

\appendix

\section{Proofs of the theoretical results}\label{appendix:theory-proofs}
We collect here the proofs of the results used in Section~\ref{sec:method}.

\subsection{Filtering loss decomposition}\label{appendix:loss-decompositions}

Here $H$ denotes conditional differential entropy.

\begin{proof}[Proof of Proposition~\ref{prop:filtering-decomposition}]
Conditioning the population loss on $s_t$ gives
\begin{equation*}
    \mathcal L_F^{(t)}(\theta_1,\psi)
    =H(u_t\mid s_t)
    +\E_{s_t}\!\left[
    \KL\!\left(
    p(\,\cdot\mid s_t)
    \,\middle\|\,
    p_{\theta_1,\psi}(\,\cdot\mid s_t)
    \right)\right].
\end{equation*}
Since $s_t$ is a deterministic function of $y_{1:t}$, we have $p(u_t\mid y_{1:t},s_t)=p(u_t\mid y_{1:t})$. The conditional mutual information therefore satisfies
\begin{equation*}
I(u_t;y_{1:t}\mid s_t)
    =\E_{y_{1:t}}\!\left[
    \KL\!\left(
    p(\,\cdot\mid y_{1:t})
    \,\middle\|\,
    p(\,\cdot\mid s_t)
    \right)\right]=H(u_t\mid s_t)-H(u_t\mid y_{1:t}).
\end{equation*}
Substitution shows that the population loss differs from the objective in \eqref{eq:filtering-loss-decomposition} by $H(u_t\mid y_{1:t})$. This term is independent of $(\theta_1,\psi)$, which proves the equivalence.
\end{proof}

\subsection{Shared-summary information relation}\label{appendix:shared-summary-information}

\begin{proof}[Proof of Proposition~\ref{prop:shared-summary}]
Applying the conditional mutual-information chain rule to $(u_k,u_{k+1})$ in the two possible orders gives
\begin{align*}
    I((u_k,u_{k+1});y_{1:k}\mid s_k)
    &=I(u_k;y_{1:k}\mid s_k)+I(u_{k+1};y_{1:k}\mid u_k,s_k)\\
    &=I(u_{k+1};y_{1:k}\mid s_k)+I(u_k;y_{1:k}\mid u_{k+1},s_k).
\end{align*}
Because $u_{k+1}\perp y_{1:k}\mid u_k$ and $s_k$ is a function of $y_{1:k}$, $I(u_{k+1};y_{1:k}\mid u_k,s_k)=0$. Equating the two expansions proves \eqref{eq:shared-information-relation}. The stated inequality follows from the nonnegativity of conditional mutual information.

If $s_k$ is filtering sufficient at time $k$, then $I(u_k;y_{1:k}\mid s_k)=0$. Hence \eqref{eq:shared-information-relation} gives
\begin{equation*}
    0=I(u_{k+1};y_{1:k}\mid s_k)+I(u_k;y_{1:k}\mid u_{k+1},s_k),
\end{equation*}
and both nonnegative terms must vanish. In particular, $I(u_k;y_{1:k}\mid u_{k+1},s_k)=0$, which is equivalent to \eqref{eq:backward-sufficiency}.
\end{proof}

\subsection{Shared recursive summary in the linear-Gaussian model}\label{appendix:linear-gaussian-construction}
\begin{proof}[Proof of Example~\ref{ex:linear-gaussian-statistic}]
Initialize $s_0^\star=(m_0,P_0)$. The Kalman update is
\begin{align*}
    m_{t\mid t-1}&=Mm_{t-1},
    &P_{t\mid t-1}&=MP_{t-1}M^\top+Q,\\
    \Sigma_t&=HP_{t\mid t-1}H^\top+R,
    &K_t&=P_{t\mid t-1}H^\top\Sigma_t^{-1},\\
    m_t&=m_{t\mid t-1}+K_t(y_t-Hm_{t\mid t-1}),
    &P_t&=P_{t\mid t-1}-K_t\Sigma_tK_t^\top.
\end{align*}

These equations determine $(m_t,P_t)$ from $(m_{t-1},P_{t-1},y_t)$ and the fixed model matrices. They therefore define a map $\Phi$, independent of $t$, such that $s_t^\star=\Phi(s_{t-1}^\star,y_t)$. The mean and the distinct entries of the covariance identify $s_t^\star$ with a vector of dimension $n_u+n_u(n_u+1)/2$. Since the filtering distribution is $\mathcal N(m_t,P_t)$, this summary is filtering sufficient.

To construct the backward transport, define
\begin{equation*}
    G_t=P_tM^\top P_{t+1\mid t}^{-1},
    \qquad
    C_t=P_t-G_tP_{t+1\mid t}G_t^\top.
\end{equation*}
Conditioning the Gaussian transition on $u_{t+1}$ gives
\begin{equation*}
    p(u_t\mid u_{t+1},y_{1:t})
    =\mathcal N\!\left(m_t+G_t(u_{t+1}-Mm_t),\,C_t\right).
\end{equation*}
The parameters depend on the observation history only through $s_t^\star$, which also verifies backward sufficiency. Under the stated assumptions, $P_t\succ0$ and $C_t=(P_t^{-1}+M^\top Q^{-1}M)^{-1}\succ0$.

Let $L_tL_t^\top=P_t$ and $\widetilde L_t\widetilde L_t^\top=C_t$ be Cholesky factorizations. For $z\sim\mathcal N(0,I_{n_u})$, the conditional affine maps
\begin{equation*}
    \mathcal T_F(z;s_t^\star)=m_t+L_tz
\end{equation*}
and
\begin{equation*}
    \mathcal T_B(z;v,s_t^\star)=m_t+G_t(v-Mm_t)+\widetilde L_tz
\end{equation*}
generate the filtering distribution and, with $v=u_{t+1}$, the backward kernel, respectively. All coefficients are determined by the conditioning variables and the fixed model matrices, so the same two maps apply at every assimilation time.

For any horizon $T$, draw the terminal state using $\mathcal T_F(\,\cdot\,;s_T^\star)$, and then draw successive states using $\mathcal T_B(\,\cdot\,;u_{t+1},s_t^\star)$, with independent standard Gaussian reference variables. The backward factorization implies that the resulting trajectory has distribution $p(u_{1:T}\mid y_{1:T})$.
\end{proof}

\subsection{Joint objective}\label{appendix:joint-objective}

\begin{proof}[Proof of Proposition~\ref{prop:population-objective}]
Multiplying the population expectation of \eqref{eq:empirical-path-objective} by $T$ gives a sum of intermediate filtering cross-entropies and a joint smoothing cross-entropy. For every $1\leq\ell<T$, the conditional cross-entropy identity gives
\begin{equation*}
    \E_{u_\ell,y_{1:\ell}}\!\left[
    -\log p_{\theta_1,\psi}(u_\ell\mid s_\ell)
    \right]
    =H(u_\ell\mid y_{1:\ell})
    +\E_{y_{1:\ell}}\!\left[
    \KL\!\left(
    p(\,\cdot\mid y_{1:\ell})
    \,\middle\|\,
    p_{\theta_1,\psi}(\,\cdot\mid s_\ell)
    \right)\right].
\end{equation*}
The corresponding identity for the joint smoothing cross-entropy is
\begin{align*}
    \E_{u_{1:T},y_{1:T}}\!&
    \left[
    -\log p_{\theta_1,\theta_2,\psi}^{\mathrm{smoothing}}(u_{1:T}\mid s_{1:T})
    \right]\\
    &=
    H(u_{1:T}\mid y_{1:T})+\E_{y_{1:T}}\!\left[
    \KL\!\left(
    p(\,\cdot\mid y_{1:T})
    \,\middle\|\,
    p_{\theta_1,\theta_2,\psi}^{\mathrm{smoothing}}(\,\cdot\mid s_{1:T})
    \right)\right].
\end{align*}
Substituting these identities shows that $T\mathcal L$ and the KL objective in Proposition~\ref{prop:population-objective} differ only by conditional entropy terms independent of $(\theta_1,\theta_2,\psi)$. The two objectives therefore have the same minimizers.
\end{proof}

\subsection{Joint and marginal smoothing errors}\label{appendix:joint-marginal-errors}

\begin{proof}[Proof of Theorem~\ref{thm:path-kl}]
Using the exact factorization \eqref{eq:path_factorization} and the approximate factorization \eqref{eq:learned-smoothing-path} at horizon $t$, their pointwise log-density ratio is
\begin{equation*}
\log\frac{p(u_{1:t}\mid y_{1:t})}
    {p_{\theta_1,\theta_2,\psi}^{\mathrm{smoothing}}(u_{1:t}\mid s_{1:t})}=
    \log\frac{p(u_t\mid y_{1:t})}{p_{\theta_1,\psi}(u_t\mid s_t)}
    +\sum_{\ell=1}^{t-1}
    \log\frac{p(u_\ell\mid u_{\ell+1},y_{1:\ell})}
    {p_{\theta_2,\psi}(u_\ell\mid u_{\ell+1},s_\ell)}.
\end{equation*}
Taking expectation under $p(u_{1:t}\mid y_{1:t})$ turns the terminal term into the terminal KL divergence. For each $\ell<t$, conditioning first on $u_{\ell+1}$ gives
\begin{equation*}
\resizebox{\linewidth}{!}{$\displaystyle
\E_{u_{1:t}\mid y_{1:t}}\!\left[
    \log\frac{p(u_\ell\mid u_{\ell+1},y_{1:\ell})}
    {p_{\theta_2,\psi}(u_\ell\mid u_{\ell+1},s_\ell)}
    \right]=
    \E_{u_{\ell+1}\mid y_{1:t}}\!\left[
    \KL\!\left(
    p(\,\cdot\mid u_{\ell+1},y_{1:\ell})
    \,\middle\|\,
    p_{\theta_2,\psi}(\,\cdot\mid u_{\ell+1},s_\ell)
    \right)\right],
$}
\end{equation*}
where the last equality uses $p(u_\ell\mid u_{\ell+1},y_{1:t})=p(u_\ell\mid u_{\ell+1},y_{1:\ell})$. Summing over $\ell$ proves the joint identity in Theorem~\ref{thm:path-kl}.

For any $1\leq k\leq t$, integrating out $u_{1:k-1}$ from the two normalized backward factorizations yields
\begin{align*}
    p(u_{k:t}\mid y_{1:t})
    &=p(u_t\mid y_{1:t})\prod_{\ell=k}^{t-1}p(u_\ell\mid u_{\ell+1},y_{1:\ell}),\\
    p_{\theta_1,\theta_2,\psi}^{\mathrm{smoothing}}(u_{k:t}\mid s_{1:t})
    &=p_{\theta_1,\psi}(u_t\mid s_t)
    \prod_{\ell=k}^{t-1}p_{\theta_2,\psi}(u_\ell\mid u_{\ell+1},s_\ell).
\end{align*}
Repeating the preceding calculation shows that the right-hand side of the marginal bound in Theorem~\ref{thm:path-kl} is the KL divergence between these two tail laws. Since $u_k$ is the projection of $u_{k:t}$ onto its first component, the data-processing inequality gives
\begin{equation*}
\resizebox{\linewidth}{!}{$\displaystyle
\KL\!\left(
    p(u_k\mid y_{1:t})
    \,\middle\|\,
    p_{\theta_1,\theta_2,\psi}^{\mathrm{smoothing}}(u_k\mid s_{1:t})
    \right)\leq
    \KL\!\left(
    p(u_{k:t}\mid y_{1:t})
    \,\middle\|\,
    p_{\theta_1,\theta_2,\psi}^{\mathrm{smoothing}}(u_{k:t}\mid s_{1:t})
    \right),
$}
\end{equation*}
which proves the marginal bound.
Applying the data-processing inequality directly to the projection $u_{1:t}\mapsto u_k$ and averaging over $y_{1:t}$ also bounds each expected marginal divergence by the expected joint divergence.
\end{proof}

\renewcommand{\thesection}{\Alph{section}}

\makeatletter
\@addtoreset{table}{section}
\@addtoreset{figure}{section}
\@addtoreset{equation}{section}
\@addtoreset{theorem}{section}
\makeatother

\renewcommand{\thetable}{\thesection.\arabic{table}}
\renewcommand{\thefigure}{\thesection.\arabic{figure}}
\renewcommand{\theequation}{\thesection.\arabic{equation}}
\renewcommand{\thetheorem}{\thesection.\arabic{theorem}}

\section{Evaluation metrics}\label{appendix:metrics}
Following~\cite{wang2025flow}, let $u_k^{\mathrm{true}}\in\mathbb R^d$ denote the realized state and let $\{u_k^{(j)}\}_{j=1}^S$ be $S=N_{\mathrm{sample}}$ samples from the conditional distribution being evaluated at each of $N_{\mathrm{eval}}$ times. The root mean squared error of the sample mean is
\begin{equation}
\operatorname{RMSE}
=\left[\frac{1}{dN_{\mathrm{eval}}}\sum_{k=1}^{N_{\mathrm{eval}}}
\left\|u_k^{\mathrm{true}}-\frac1S\sum_{j=1}^S u_k^{(j)}\right\|_2^2\right]^{1/2}.
\end{equation}
The following score compares the empirical posterior distribution with the point mass at the realized state. For a positive-definite kernel $\kappa$, the squared maximum mean discrepancy score at time $k$, denoted by $\operatorname{MMD}_k^2$, is
\begin{equation}
\operatorname{MMD}_k^2
=\frac{1}{S^2}\sum_{i,j=1}^S\kappa\bigl(u_k^{(i)},u_k^{(j)}\bigr)
-\frac{2}{S}\sum_{j=1}^S\kappa\bigl(u_k^{(j)},u_k^{\mathrm{true}}\bigr)
+\kappa\bigl(u_k^{\mathrm{true}},u_k^{\mathrm{true}}\bigr),
\end{equation}
and $\operatorname{MMD}^2=N_{\mathrm{eval}}^{-1}\sum_{k=1}^{N_{\mathrm{eval}}}\operatorname{MMD}_k^2$. We use the Gaussian kernel $\kappa(x,y)=\exp\{-\|x-y\|_2^2/(2\gamma^2)\}$ with $\gamma=2$. For state component $\ell$, the continuous ranked probability score is
\begin{equation}
\operatorname{CRPS}_{k,\ell}
=\int_{-\infty}^{\infty}\left[
\mathbf 1\bigl(u_{k,\ell}^{\mathrm{true}}\leq x\bigr)
-\frac1S\sum_{j=1}^S\mathbf 1\bigl(u_{k,\ell}^{(j)}\leq x\bigr)
\right]^2\,\mathrm dx,
\end{equation}
and $\operatorname{CRPS}=(dN_{\mathrm{eval}})^{-1}\sum_{k=1}^{N_{\mathrm{eval}}}\sum_{\ell=1}^d\operatorname{CRPS}_{k,\ell}$. Each score is first computed for an individual test trajectory, and the reported values are averages over all test trajectories.

\section{Additional results for Burgers' equation}\label{appendix:additional-results}\label{appendix:burgers-details}\label{appendix:additional-burgers}\label{appendix:burgers-trajectories}\label{appendix:burgers0.01}
The left panels of Figure~\ref{fig:burgers-additional} supplement the spatial comparison in Figure~\ref{burgers0.05st} with the temporal evolution at three spatial locations for $\nu=0.05$ and $r^2=0.25$. The filtering and smoothing means follow the main temporal variations over the $200$-step test trajectory, corresponding to physical time $1$.

We also consider the lower viscosity $\nu=0.01$, using the observation model and training settings of Section~\ref{burgers}. The filtering metrics in Table~\ref{burgers0.01vs} are larger than at $\nu=0.05$, while AFSF yields lower values than FBF at every observation-noise level. The right panels of Figure~\ref{fig:burgers-additional} show the filtering mean and uncertainty bands at $r^2=0.25$.

\begin{table}[ht]
	\caption{Filtering metrics for Burgers' equation with $\nu=0.01$.}\label{burgers0.01vs}
	\centering\footnotesize
	\begin{tabular}{lcccccc}
		\toprule
		& \multicolumn{2}{c}{RMSE}
		& \multicolumn{2}{c}{$\mathrm{MMD}^2$}
		& \multicolumn{2}{c}{CRPS}\\
		\cmidrule(lr){2-3}\cmidrule(lr){4-5}\cmidrule(lr){6-7}
		& AFSF & FBF & AFSF & FBF & AFSF & FBF \\
		\midrule
		$r^2 = 0.01$ & 0.1140&0.1152&0.1501&0.1521&0.0588&0.0594\\
		$r^2 = 0.04$ & 0.1270&0.1301&0.1830&0.1897&0.0686&0.0699\\
		$r^2 = 0.09$ & 0.1380&0.1403&0.2120&0.2169&0.0754&0.0766\\
		$r^2 = 0.16$ & 0.1467&0.1507&0.2360&0.2454&0.0807&0.0827\\
		$r^2 = 0.25$ & 0.1546&0.1640&0.2581&0.2833&0.0852&0.0900\\
		\bottomrule
	\end{tabular}
\end{table}

\begin{figure}[ht]
 \centering
 \includegraphics[width=.48\textwidth]{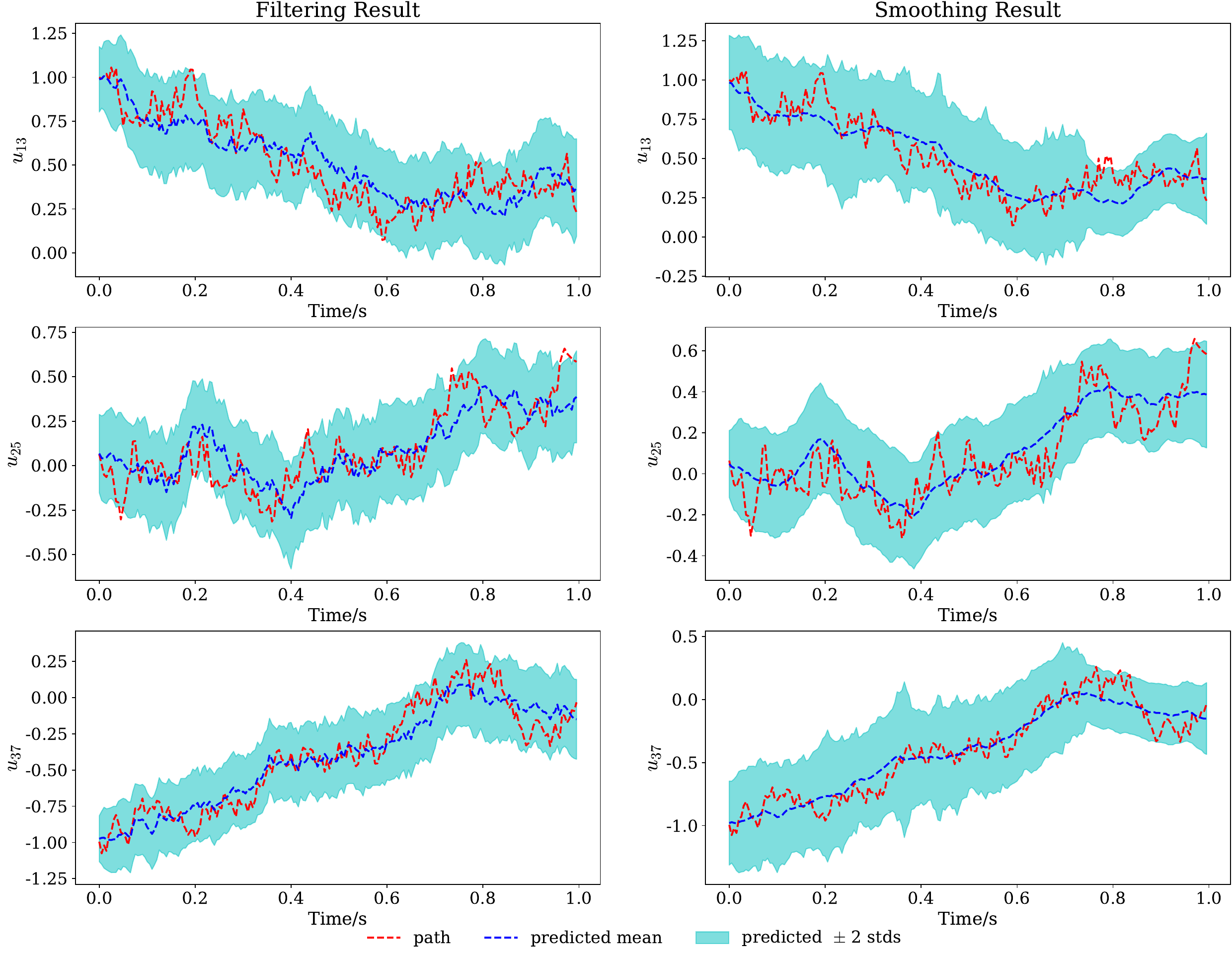}
 \includegraphics[width=.48\textwidth]{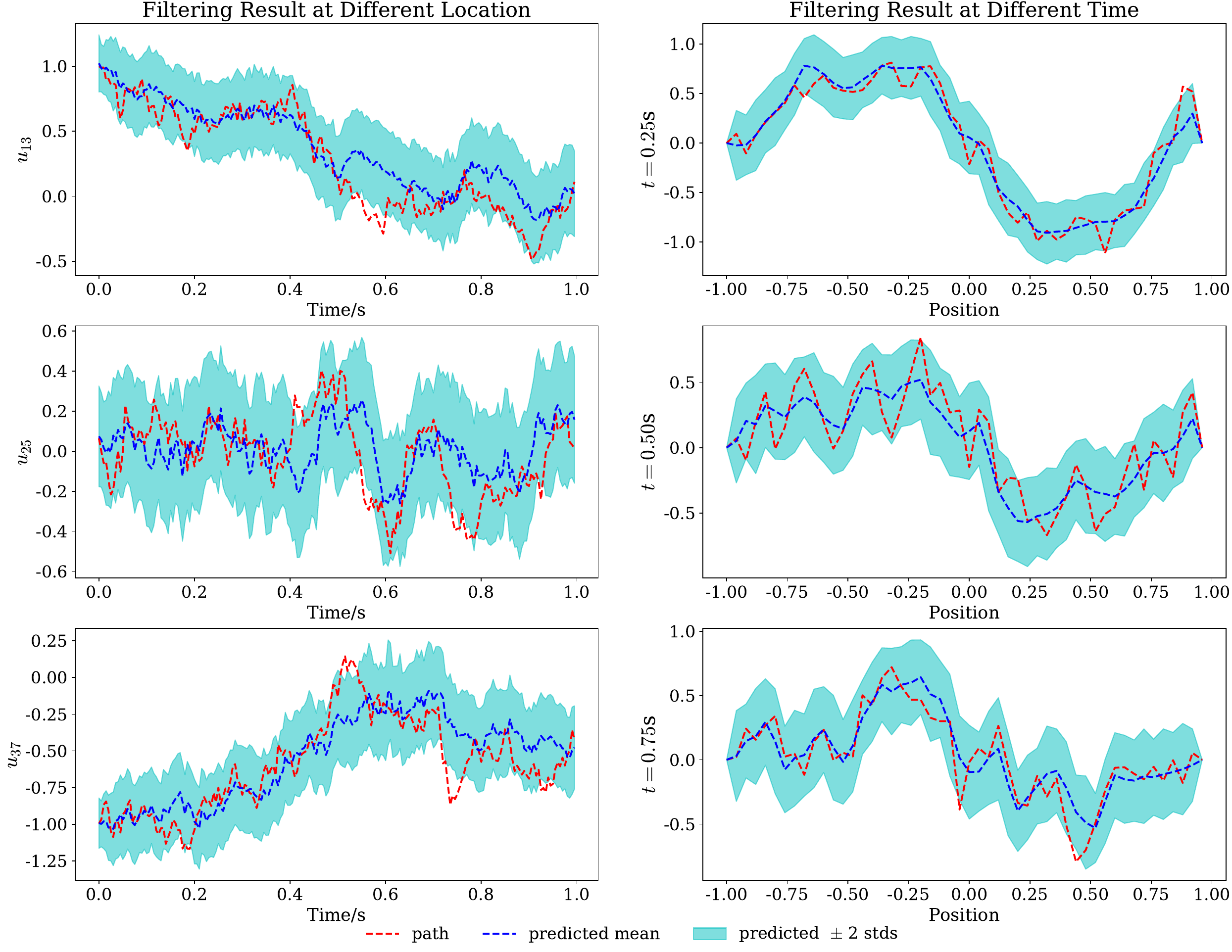}
 \caption{Realized states, posterior means, and pointwise two-standard-deviation bands for Burgers' equation with $r^2=0.25$. The first two columns show filtering and smoothing, respectively, at three spatial locations for $\nu=0.05$. The last two columns show filtering for $\nu=0.01$, with temporal evolution at three spatial locations and spatial profiles at three times, respectively.}
 \label{fig:burgers-additional}
\end{figure}

\FloatBarrier

\bibliographystyle{siamplain}
	\bibliography{ref}
\end{document}